\documentclass[11pt]{article}
\usepackage{amssymb}
\usepackage{amsfonts}
\usepackage{amsmath}
\usepackage{amsthm}
\usepackage[nohead]{geometry}
\usepackage[onehalfspacing]{setspace}
\usepackage[bottom]{footmisc}
\usepackage{indentfirst}
\usepackage{dcolumn}
\usepackage{endnotes}
\usepackage{subcaption}
\usepackage{graphicx}
\usepackage{hyperref}
\usepackage{xr-hyper}
\usepackage{xr}
\usepackage{multirow}
\usepackage[normalem]{ulem}
\usepackage{algorithm}
\usepackage[noend]{algpseudocode}
\usepackage{bbm}
\usepackage{caption}
\usepackage{subcaption}
\graphicspath{{figures/}}
\usepackage{float}
\usepackage{subcaption}
\usepackage{booktabs,caption,fixltx2e}
\usepackage[flushleft]{threeparttable}
\usepackage{rotating}
\usepackage[round]{natbib}
\usepackage{tikz}
\usepackage{appendix}
\usepackage[colorinlistoftodos]{todonotes}
\usepackage{pdflscape}
\usepackage{booktabs}
\usepackage{siunitx}
\usepackage{caption}
\usepackage{epstopdf}
\usepackage{multirow}
\usepackage{color,soul}
\usepackage{booktabs,caption}
\usepackage[normalem]{ulem}
\usepackage{enumitem}
\usepackage[flushleft]{threeparttable}
\usepackage{float}
\usepackage[version=3]{mhchem}
\usepackage{pgf}
\usepackage{tikz}                   % tikz graphics
\usepackage{pgfplots}
\usetikzlibrary{arrows,calc}
\usepgfplotslibrary{dateplot}
\usepackage{verbatim}
\usepackage{breqn}
\usepackage{appendix}
\pgfplotsset{compat=1.12}
\usetikzlibrary{fillbetween}
\usepackage{filecontents}
\usepackage{multirow}
\usepackage{apalike}
\usepackage{amsfonts}
\usepackage{threeparttable}
\usepackage[normalem]{ulem}

\usepackage{algorithm}
\usepackage[noend]{algpseudocode}

\newenvironment{sbmatrix}[1]
 {\def\mysubscript{#1}\mathop\bgroup\begin{bmatrix}}
 {\end{bmatrix}\egroup_{\textstyle\mathstrut\mysubscript}}

\newcommand{\eqtext}[1]{\ensuremath{\stackrel{\text{#1}}{=}}}
\newcommand{\geqtext}[1]{\ensuremath{\stackrel{\text{#1}}{\geq}}}

\makeatletter
\def\BState{\State\hskip-\ALG@thistlm}
\makeatother
\pgfplotsset{
	/pgfplots/area cycle list/.style={/pgfplots/cycle list={%
			{black,fill=yellow!20!white,mark=none},%
			{black,fill=yellow!40!white,mark=none},%
			{black,fill=red!20!white,mark=none},%
			{black,fill=red!40!white,mark=none},
			{black,fill=red, mark = none}
		}
	},
}

\usepackage{flexisym}
\makeatletter
\def\@biblabel#1{\hspace*{-\labelsep}}
\makeatother
\usepackage{authblk}% http://ctan.org/pkg/authblk

\DeclareMathAlphabet      {\mathbfit}{OML}{cmm}{b}{it}

\makeatletter
\def\BState{\State\hskip-\ALG@thistlm}
\makeatother

\usepackage[normalem]{ulem}
\newtheorem{theorem}{Theorem}
\newtheorem{remark}{Remark}

\newtheorem{prop}{Proposition}

\newtheorem{lemma}{Lemma}

\newtheoremstyle{case}{}{}{}{}{}{:}{ }{}
\theoremstyle{case}

\DeclareMathOperator*{\argmin}{arg\,min}

\usepackage{authblk}

\title{ Fused-Lasso Regularized  Cholesky Factors of  Large Nonstationary Covariance Matrices of Longitudinal Data}
\author[$\dagger$]{Aramayis Dallakyan \thanks{{Correpondence to: Aramayis Dallakyan, 3143 TAMU, Department of Statistics, College Station, TX 77843, USA. \\
E-mail: dallakyan1988@tamu.edu}}}
\author[$\dagger$]{Mohsen Pourahmadi}
\affil[$\dagger$]{Department of Statistics, Texas A\&M University, College Station, TX 77843, USA}
\date{}

\begin{document}
	
\maketitle
\sloppy
\singlespacing

\begin{center}
\textbf{Abstract}
\end{center}
\noindent Smoothness of the  subdiagonals of the Cholesky factor of large  covariance matrices is closely related to the degrees
 of nonstationarity of autoregressive models for time series and longitudinal data. Heuristically, one expects for a nearly stationary covariance matrix the entries in each subdiagonal of the Cholesky factor of its inverse to be nearly the same in the sense that  sum of  absolute values of successive terms is small.
Statistically such smoothness is achieved by regularizing each subdiagonal using  fused-type lasso penalties. We rely on the standard Cholesky factor as the new parameters within a regularized normal likelihood setup which guarantees:
(1) joint convexity of the likelihood function,
(2) strict convexity of the likelihood function restricted to each subdiagonal even when $n<p$, and (3) positive-definiteness of the estimated covariance matrix.
 A block coordinate descent algorithm, where each block is a subdiagonal, is proposed and its convergence is established under mild conditions. Lack of decoupling of the penalized likelihood function into a sum of functions involving individual subdiagonals gives rise to some computational challenges and  advantages relative to two recent algorithms for sparse estimation of the Cholesky factor which decouple row-wise. Simulation results and real data analysis show the scope and good performance of the proposed methodology. 
 %\sout{In particular, our analysis of call center data  illustrates the superior predictive performance of our method compared to the existing sparse Cholesky approaches.}

\noindent

\vspace{0.5 cm}
\strut


\vspace{0.25 cm}
\noindent\textbf{MOS subject classifications:} 62A09, 60G99  
\strut
\vspace{0.5 cm}
\thispagestyle{empty}

\pagebreak
%\onehalfspacing
\doublespacing
\setlength{\parskip}{.85mm plus.25mm minus.25mm}

%\tableofcontents
%\newpage

%%%%%%%%%%%%%%%%%%%%%%%%%%%%%%%%%%%%%%%%%%%%%%%%%%%%%%%%%%%%%%%%%%%%%%%%%%%%%%%%%%%%%%%%%%%%%%%%%%%%%%%%%%%%%%%%%%%%%%%%%%%%%%%%%%%%%%%%%%%5
%%%%%%%%%%%%%%%%%%%%%%%%%%%%%%%%%%%%%%%%%%%%%%%%%%%%%%%%%%%%%%%%%%%%%%%%%%%%%%%%%%%%%%%%%%%%%%%%%%%%%%%%%%%%%%%%%%%%%%%%%%%%%%%%%%%%%%%%%%%%
\section{Introduction} \label{s:s1}

A  salient feature of stationary time series analysis is its reliance on  the Cholesky decomposition
 to model temporal dependence and the dynamics. Important examples include moving average models (Cholesky decomposition of a covariance matrix), autoregressive (AR) models (Cholesky decomposition of an inverse covariance matrix), ARMA models
 in the time-domain \citep{ansley1979}, see  \cite{dai2004,rosen2007} for explicit use of the Cholesky factors in the spectral-domain. For nonstationary time series  the focus has been  on (time-)varying coefficients AR models \citep{gabriel1962, rao1970, kitagawa1985, dahlhaus1997,zimmerman2010}.

Recently, a similar dichotomy is taking roots in the modern multivariate statistics and machine
learning  where the focus is  on either estimation of  large covariance or inverse
covariance matrices  of longitudinal data using Cholesky decomposition. Whereas the entries of a covariance matrix quantifies pairwise or marginal dependence, those of
the precision or inverse covariance matrix specifies multivariate relationships among the variables
in a $p$-dimensional random vector $X =(X_1,\dots,X_p)^t \in R^p$
with a positive-definite covariance matrix $\Sigma$. More precisely, when $X$ follows a Gaussian distribution a zero off-diagonal entry of $\Omega = (\Omega_{j,k})=\Sigma^{-1}$
or $\Omega_{j,k}=0$
implies that $X_j$ and $X_k$ are conditionally independent given all other variables \citep{whittaker1990}.
When the number of observations $n$ is less than
the number of variables $p$, it is reasonable to impose structure or regularize
$\Omega$ directly in the search for sparsity \citep{Banerjee2007,Friedman2008}, see \citet{pourahmadi2013} for an overview.

The use of the modified Cholesky decomposition of $\Omega$ was advocated in \citep{pourahmadi1999,Wu2003},
\citet{Huang2006} and \citet{levina2008}  for parsimony (GLM-based) and sparse (regularized) estimation  of its Cholesky factor and hence the precision matrix. Recall that the standard and modified Cholesky factors of a positive-definite precision matrix are defined and connected by
\begin{equation}\label{eq:e1}
\Omega =L^tL= T^t\Lambda^{-1}T, \; L= \Lambda^{-1/2}T,%\;\;\mbox{for}\;\;i \leq j,
\end{equation}
where $L=(L_{i,j})$ is a unique lower triangular matrix with positive diagonal entries  and
 $T=(\phi_{i,j})$ is a unit lower triangular matrix with diagonal entries equal to 1, $\Lambda=\mbox{diag}(\sigma^2_1,\dots,\sigma^2_p)$ is a diagonal
matrix with positive diagonal entries. For time series and longitudinal data  the entries in each row of $T$ have the useful   interpretation
as the regression coefficients and each diagonal entry of $\Lambda$ as the  variance of the residual $\varepsilon_t$ of regressing a
variable on its preceding variables:
\begin{equation}\label{AR}
X_t=\sum_{j=1}^{t-1}\phi_{tj}X_{t-j}+\varepsilon_t,\;\; t=1,2,\dots, p, \ \ \phi_{11}=0.
\end{equation}
The genesis of this representation and interpretation of the coefficients for stationary processes can be traced to the rise of finite-parameter AR models in 1920's
 \citep[Section 1.2]{pourahmadi2001}; \citep{ansley1979}. For example, a stationary AR model of order $p$ is closely related to a $p$-banded lower triangular matrix where all entries of its first subdiagonal are the same and equal to the negative of the lag-1 AR coefficient, and so on.
Heuristically, one expects for a nearly
stationary (Toeplitz) covariance matrix the entries in each subdiagonal of the Cholesky factor
of the inverse covariance matrix to be nearly the same in the sense that sum of absolute values of its successive
terms is small. Important examples of  mild
departures from stationarity are  locally stationary  \citep{dahlhaus1997}
and  piecewise stationary \citep{adak1998,davis2006} processes where in the latter  the  subdiagonals could be certain step functions. Figure~\ref{fig:cplot} illustrates the adverse effect of learning a genuinely nonstationary covariance matrix of the cattle  data \citep{kenward1987} using a  (misspecified) stationary AR model.

 We emphasize the time-varying nature  of the  coefficients
  $\phi_{tj}$ in (\ref{AR}) for fixed $j$  by  using a doubly indexed triangular array $X_{t,p}$ \citep{dahlhaus1997}  and writing it more generally as
 \begin{equation}\label{ARS}
X_{t,p}=\sum_{j=1}^{p_t}\alpha_{j}(\frac tp)X_{t-j,p}+ \sigma(\frac tp)\varepsilon_t,\;\; t=1,\dots, p,
\end{equation}
where $ 0\le p_t\le p$, $\alpha_{j}(u$) and $\sigma(u)$ are smooth functions of the rescaled time  $u=\frac tp \in [0, 1]$ and $\varepsilon$'s are
 i.i.d. random variables with mean zero and variance one. This  rescaling enables one to  view the (sub)diagonals of  $T$ and $\Lambda$ as realizations of smooth functions  (see Figure~\ref{fig:tvar}) and brings the estimation problem within the familiar nonparametric  infill asymptotic setup where one observes the smooth functions $\alpha_{j}(u$) and $\sigma(u)$ on a finer grids for a larger $p$.
  Interestingly, choosing
 $\alpha_{j}(u$) and $\sigma(u)$ as functions of bounded variation guarantees that, under mild conditions, the solutions of (\ref{ARS}) are locally stationary processes \citep[Proposition 2.4]{dahlhaus2009}.

 In addition to its profound conceptual impact on time series analysis  \citep{dahlhaus2012},
  the functional view of (\ref{ARS}) for longitudinal data has been a major  source of inspiration    for  nonparametric estimation of the subdiagonals of $T$, see \citet{Wu2003} and  \citet{huang2007}. Furthermore, within the
  smoothing spline ANOVA
framework,  \citet{blake2018} treats the AR coefficients $\phi_{tj}, t>j$ as a bivariate smooth function and decomposes it in the stationary direction of the lag $\ell =t-j$ and the nonstationary (additive) direction $m=\frac {t+j}2$ and a possible interaction term. Then, regularizing the nonstationary  direction more heavily amounts to shrinking the covariance estimator toward the more parsimonious and desirable stationary structures.

 In the longitudinal data setup, with a sample $X_1,\cdots, X_n \sim N_p(0,\Sigma)$ and   the sample covariance matrix $S=n^{-1}\sum_{i=1}^nX_iX_i'$, its log-likelihood function  $\ell(\Omega)=\mbox{tr}(\Omega S)-\log|\Omega|$ was used for penalized likelihood
 estimation of the parameters $(T,\Lambda)$ in \citet{Huang2006}, see also
   \citet{levina2008} and \citet{khare2016} for a comprehensive review. The lack of convexity of the likelihood in $(T, \Lambda)$ was noted first in  \citet{khare2016} and \citet{yu2017}. They ensure  convexity by reparameterizing the likelihood in terms of the standard Cholesky factor $L$ rather than the customary $(T,\Lambda)$-parametrization. While the last identity in (\ref{eq:e1}) reveals that
    $T$ and $L$  share  the same sparsity patterns, the connection between the degree of smoothness of their subdiagonals is a bit more complicated and controlled by the boundedness and smoothness of  the diagonal entries of $\Lambda$ (see Lemma~\ref{l:TL}).

   This paper is concerned with  smoothness through regularizing the subdiagonals of the Cholesky factor $L$ of  $\Omega$ using the
 fused Lasso penalties \citep{Tibshirani2005} as an alternative to their smooth (nonparametric) estimation.
More specifically, using  the family of fused lasso penalty functions on the subdiagonals we propose a novel \textit{smooth Cholesky (SC) algorithm} for estimating the subdiagonals of $L$ and hence the (inverse) covariance matrix via a block coordinate decent algorithm. The SC objective function is convex in $L$, and compared to the recent algorithms in \citet{khare2016} and \citet{yu2017} when $n <<p$, the update
of each block is obtained by solving a strictly convex optimization problem.  We establish the convergence of the iterates to stationary points of the objective function, and elaborate on the connection between  the smoothness of  the subdiagonals of $L$  and those of  $T$ under the  assumption of boundedness of the diagonal entries of  $\Lambda$.

%We note that in the high-dimensional setting of $n <<p$, each iteration involves optimization of a strictly convex function, even though the objective function of the method is not strictly convex and the convergence properties of the algorithm can be complicated.
%\end{enumerate}f

The remainder of the paper is organized as follows. Section~\ref{s:s2} introduces the SC algorithm and studies its
convergence and computational complexity. Section~\ref{s:TandL}  establishes the connection between smoothness of subdiagonals of $L$ and $T$.
% In Section~\ref{s:s2.x}, we discuss the relationship between SSC and the previous proposals.
Section~\ref{s:s3} illustrates the performance of the SC methodology through simulations and real data analysis, %Section~\ref{s:s4} provides high-dimensional asymptotic consistency of the SSC method.
%%%%%%%%%%%%%%%%%%%%%%%%%%%%%%%%%%%%%%%%%%%%%%%%%%%%%%%%%
and demonstrates its ability to model and detect the
smoothness of  the subdiagonals of Cholesky factor. Consequently, estimation of the covariance matrix
and its role in forecasting the future calls in a call center are investigated. Appendices contain proofs of the result in the paper and some additional simulations. All appendices are placed in Supporting Information.

% d We introduce the operator $vsub(\cdot)$, which sequentially stacks the subdiagonals and denote $s=vsub(L)$. It is easy to see that, there exist selection matrix $C$, such that $s=Kv$.
%the submatrix of $M$ using the indices in $I_1$. In particular, for the vector $x= vec(L)$, $L^i, i=1,\ldots, p-1$ is
%the vector of $i$th subdiagonal, and $L^{-i}$ is the vector of the remaining nonredundant subdiagonals and diagonal of $L$.

 Our SC algorithm and the corresponding methodology for longitudinal data can be specialized to the setup of  a long stretch of a single
stationary time series, namely for $n=1$ and $p$ large. To this end, banded estimates of Toeplitz covariance matrices and properties of the corresponding optimal linear predictors are studied in  \citet{wu2009, bickel2011} and \citet{politis2010,mcmurry2015}. For covariance estimation and prediction of locally stationary processes, see \citet{das2020}.

In the rest of this section, we introduce notation used throughout the paper. For a vector $x=(x_1,\dots,x_p) \in \mathcal{R}^p$, we
define its norm $\|x\|_q = (\sum_{i=1}^q |x_i|^q)^{1/q}$ for $q \geq 1$. We
denote by $\mathcal{L}_p$ the space of all lower triangular matrices with positive diagonal
elements. Given a $p \times p$ lower-triangular matrix $L$, the $p^2 \times 1$ vector $V= (v_i)= vec(L)$ is its standard vectorization formed by stacking up its column vectors including the zero (redundant) entries. Each vector of (sub)diagonal  entries of $L$
corresponds to those from $V$ with the following set of indices:
$$I_j = \{k(p+1)+j+1:k=0,\dots,(p-j-1)\}, j=0, 1,\ldots, p-1,$$
so that $I_0$ corresponds to the main diagonal entries,
 $L^j=V_{I_j}=(v_i)_{i \in I_j}$  is the $|I_j|$-subvector of the $j$th subdiagonal entries. We denote by $L^{-j}=(v_i)_{\{i \in I_k, k \neq j\}}$ a vector of diagonal and  subdiagonals, except for the $j$th subdiagonal. For simplicity in notation, we  replace $I_j$ by $j$ so that for a given $p^2 \times p^2$ matrix $A$ and index sets $I_j, I_k$, $A_{\cdot j}$ denotes the $p^2 \times |I_j|$ submatrix with column indices selected from $I_j$, and  $A_{jk}$ is the $|I_j|\times |I_k|$ submatrix with rows and columns of $A$ indexed by $I_j$ and $I_k$, respectively.

\section{ The  Smooth Cholesky Algorithm} \label{s:s2}

In this section, we develop the SC algorithm for a convex penalized likelihood function using fused-type Lasso  penalties on the  subdiagonals of the standard Cholesky factor. Such  penalties are bound to induce various degrees of sparsity and smoothness on the subdiagonals, but our main focus is on smoothness. The objective functions turn out to be conditionally separable. Computational and statistical properties of a block coordinate descent algorithm for its minimization are studied.

\subsection{The Gaussian-Likelihood and  Fused Lasso Penalties}
  Let $\ell(\Omega)$ be the Gaussian log-likelihood function for a sample of size $n$ from a zero-mean normal distribution with the precision matrix $\Omega$.
   Its convexity  is ensured by reparametrizing it in terms of the standard Cholesky factor $L$, see  \citet{khare2016} and \citet{yu2017}.
   More precisely, we consider
\begin{equation}\label{eq:e2}
Q(L) = tr(L^{t}LS)- 2\log |L| + \lambda P(L),
\end{equation}
where $P(L)$ is a convex penalty function. There are two recent important choices of $P(L)$ designed to induce sparsity
 in the rows of the Cholesky factor.

 The method of
Convex Sparse Cholesky Selection (CSCS) of $L$  in \citet{khare2016} employs the penalty
 $P(L)=\|L\|_1$. The ensuing objective function turns out to be
 jointly convex in the (nonredundant) entries of $L$, bounded away from $-\infty$ even if
$n < p$; but it is not strictly convex in the high-dimensional case. A cyclic coordinatewise
minimization algorithm is developed in \citet{khare2016} to compute $L$.
Note that once $L$ is computed using the CSCS or other methods considered here, then one can compute  $(T,\Lambda)$, the
(inverse) covariance matrix $\Sigma$ and $\Omega$. Sparsity of $\Omega$ is not guaranteed since
the sparsity pattern of the estimated $L$ in \citet{khare2016}, as in \citet{Huang2006} and \citet{shojaie2010},  has no particular structure.  Fortunately, a more structured sparse $L$ which guarantees sparsity of the precision matrix is developed in  \citet{yu2017}. Their
  hierarchical sparse Cholesky (HSC) method  relies on the hierarchical group penalty $P(L)=\sum_{r=2}^p\sum_{l=1}^{r-1} (\sum_{m=1}^l w_{lm}^2L_{rm}^2)^{1/2}$ where the $w_{lm}$'s
are   quadratically decaying  weights. The HSC method has the goal of learning the local dependence among the variables and leads to a more structured sparsity with a contiguous stretch of zeros in each row away from the main diagonal. Its flexibility is similar to that of the nested lasso  in \citet{rothman2010}. \citet{yu2017} relies on an alternating direction method of multipliers (ADMM) approach to compute $L$.
 Computationally, both  penalty functions  lead to a
decoupling of the above objective function into
$p$ separate and parallelizable optimization problems each involving a separate row of $L$.

For the SC algorithm developed in this paper, we employ a number of \textit{fused lasso} penalty functions on the Cholesky factor or its subdiagonals. However, unless stated otherwise the phrase  \textit{fused lasso} refers to
\[ P(L)=\sum_{i=0}^{p-1}P_{\nabla}(L^{i}),\; \  P_{\nabla}(y)=\sum_{j=2}^{p}|y_{j}-y_{j-1}|,\; y \in \mathcal{R}^{p }, \]
based on the $\ell_1$-norm of the first differences. Note that this  is slightly different  from  the more general \textit{sparse fused lasso} penalty function in \cite{Tibshirani2005} and \cite{tibshirani2011} which is of the form $$\lambda_1\sum_{j=1}^{p}|y_i|+\lambda_2P_{\nabla}(y).$$
The latter includes an additional lasso penalty term to achieve sparsity on top of smoothness of the subdiagonals. Note that our usage of  \textit{fused lasso} is more in the spirit of the total variation penalty in \citet{rudin1992}.

When   higher-order smoothness of the subdiagonals is desirable, then it is natural to  penalize sum
 of higher-order differences such as $\|D_2y\|_1$, the $\ell_1$-trend filtering \citep{kim2009}, and $\|D_2y\|^2_2$
  \citep{hodrick1997}, referred to as H-P hereafter, where
$D_2$ is the  matrix of second-order differences:
$$D_2=  \begin{bmatrix} -1 & 2 & 1 & \hdots & 0 & 0 & 0\\ 0 & -1 & 2 & \hdots & 0 & 0 & 0 \\
		\hdots \\ 0 & 0 & 0 & \hdots & -1 & 2 & 1  \end{bmatrix}.$$
 For other  higher order difference matrices belonging to the family of  generalized lasso
penalties, see \cite{Tibshirani2005};\cite{tibshirani2011}.

%The objective function (\ref{eq:e2})
%\begin{equation}\label{eq:e3}
%Q(L) = tr(L^{t}LS)- 2\sum_{i=1}^{p}logL_{i,i}+ \lambda P(L^{\eta}),
%\end{equation}
 %where $P(L^{\eta})$ is a convex penalty function  regularizes the subdiagonals of $L$.

 \subsection{The Conditionally Separable Convex Objective Function} \label{s:ssc}
 We express  the objective function  (\ref{eq:e2})  as the sum of $p$ quadratic functions each involving  distinct
  (sub)diagonals of $L$ (given the others), so that it is conditionally separable. This is in sharp contrast to
  the objective functions in \citet{khare2016}
 and \citet{yu2017} which decouple over the rows of the matrix $L$ with nice computational consequences. Nevertheless, our
  objective function is  jointly convex in $L$, and  strictly convex when $n < p$.

   Let $B = S \otimes I_p$ be the Kronecker product of the sample covariance  matrix  from a sample of size $n$ and the identity matrix.
 The structure of the matrix $B$ and the $(p-i)\times (p-j)$  submatrices $B_{ij},\, 0 \leq i,j \leq p-1$, introduced in the proof of
 the following Lemma play a vital role in proving properties of our SC algorithm.

\begin{lemma} \label{l:l1}
For the lower triangular matrix $L$ it holds that:
\begin{enumerate}
\item[(a)] The first term in (\ref {eq:e2}) can be rewritten as

\begin{equation} \label{eq:l1}
tr(LSL^{t})=V^t(S \otimes I_p)V = \sum_{i = 0} ^ {p -1} \sum_{j = 0} ^{p - 1} L^i B_{ij} L^j
% \sum_{i=0}^{p-1}q_{i}(L^i|L^{-i}) %q_0(L^0)+,
\end{equation}
 %where for $0 \leq i \leq p-1$, the functions
%are quadratic in $L^i$.
 \item[(b)] The objective function $Q(L)$ is \textit{conditionally separable} in that
 \begin{equation} \label{eq:ssc}
Q(L) = \sum_{i=0}^{p-1}Q_{i}(L^i|L^{-i}),
\end{equation}
where  for $i= 0, 1,\dots,p-1$ and  fixed $L^{-i}$,
\begin{equation} \label{eq:ssc1}
Q_{i}(L^{i}|L^{{-i}}) = q_{i}(L^{i}|L^{-i}) + \lambda P_{\nabla}({L^{i}}),\; Q_{0}(L^0|L^{-0})=q_0(L^0|L^{-0}) -2\sum_{j=1}^{p} \log L^0_{j}
\end{equation}
and
\begin{equation} \label{eq:l11}
  q_{i}(L^{i}|L^{-i})=(L^{i})^t B_{ii}L^{i} + (L^{i})^t(\sum_{\substack{j \neq i}}B_{ij}L^{j}), %q_0(L^{0})=(L^{0})^t B_{00}L^{0}  ,\;\;\ +B_{i0}L^{0},
\end{equation}

%
%\end{equation}
\item[(c)] $Q_i(\cdot)$'s are strictly convex in $L^i$ even when $n < p$.
\end{enumerate}
\end{lemma}

A proof of the lemma  is provided in the Appendix.
Parts (a) and (b)  are fundamental for constructing our SC algorithm in the spirit of  the
 coordinate descent algorithm in
  \citet[Lemma 2.3]{khare2016}.  However, since our objective function is not separable over the subdiagonals,  the details of the proof of our block coordinate descend algorithm differ considerably from those in \citet{khare2016}.

\subsection{ A Block Coordinate Descent Algorithm}

In this section, relying on the conditional separability as expressed in (\ref{eq:ssc}) we minimize $Q(L)$ using a
block coordinate descent algorithm
where each block corresponds to a subdiagonal of $L$ given the values of the others. The minimization of $Q(L)$ is done
\textit{sequentially} over the summands $Q_i(\cdot)$,  $0 \leq i \leq p-1$.
%In this section, we introduce an algorithm to subdiagonally minimize the convex objective function $Q(L)$. This procedure is novel, compare to procedures offered in existing literature, where the estimation of the objective function is implemented by iterating over the columns \citep{Banerjee2007, Friedman2008} or separating over the rows \citep{khare2016,yu2017}.
In this sense, our SC algorithm is  different from the recent approaches in covariance estimation where the objective functions are
either minimized by iterating over the columns of a covariance matrix \citep{Banerjee2007, Friedman2008} or the
rows of its Cholesky factor \citep{khare2016,yu2017}. However, it inherits some of the desirable convergence
properties of the latter two algorithms even though their optimization problems decouples
 into $p$ parallel problems over the rows of the matrix $L$.

The following two generic functions
stand for  the objective function restricted to each (sub)diagonal: %{\bf Why use these functions? Motivations?}:
\begin{equation} \label{eq:h0}
h_0(x|y_0) = 2x^t y_0 + x^tC_0x  -2 \sum_{j=1}^{p-1}\log x_j
\end{equation}
and
\begin{equation} \label{eq:h1}
h_i(x|y_i) = 2x^t y_i + x^t C_ix  + \lambda \|Dx\|_1,
\end{equation}
%$h_0$ is from $\mathcal{R}^p_+$ to $\mathcal{R}$ and $h_i$ is from $\mathcal{R}^{p-i}$ to $\mathcal{R}$ for $i =1,\dots, p-1$. Here
where $C_i=B_{ii}$ is a  diagonal matrix introduced in Lemma 1, and $y_i = \sum_{j\neq i}B_{ij}L^j ,\; 0 \leq i \leq p-1$ is a $(p-i) \times 1$ vector. Note that the function $h_0$ is from $R^p_+$ to $R$ and $h_i$ is from $R^{p-i}$ to $R$ for $1 \leq i \leq p-1$. These functions are simpler than those in \citet[equation (2.8)]{khare2016} since  the matrices $C_i$ are diagonal with positive diagonal entries so that for a fixed vector $y_i$, $h_i$'s are strictly convex functions (Lemma~\ref{l:l1}).
 We note that a block coordinate descent algorithm which sequentially optimizes $h_i$ with respect to each $L^i$ will also optimize the objective function $Q(L)$. %, see for example Theorem 1 in \cite{khare2016}.

 Consider  the global minimizers of $h_0$ and $h_i$:
%The block coordinate descent algorithm for optimizing $h_i,\; 0\leq i \leq p-1$ is build up from the functions $H_0(\cdot|y_0): R^p_{+} \rightarrow R^p_{+} $ and $H_i(\cdot|y_i): R^{p-i} \rightarrow R^{p-i}$ for $1 \leq i \leq p-1$, which returns the global minimum of $h_i$.
\begin{equation} \label{eq:H0}
x^*_0 = \argmin_{x \in \mathcal{R}^p_+} h_0(x|y_0)\;\;\mbox{and}\;\; x^*_i = \argmin_{x \in \mathcal{R}^{p-i}} h_i(x|y_i).
\end{equation}
%\begin{equation} \label{eq:H0}
%H_0(x|y_0) = \inf_{z \in \mathcal{R}^p_+} h_0(x|y_0)\;\;\mbox{and}\;\; H_i(x|y_i) = \inf_{z \in \mathcal{R}^{p-i}} h_i(x|y_i).
%\end{equation}
%Next, we show that the function $H_0$ can be computed in a closed form and provide more information about the computation of functions $\{H_i\}^{p-1}_{i=1}$.
Next, we show that the vector $x^*_0$ has a closed-form and provide  methods to compute $\{x^*_i\}^{p-1}_{i=1}$
for various members of the fused-type Lasso family.

% correspond to solutions of fused lasso problems \citep{tibshirani2011} formulated for each subdiagonal.
%\begin{lemma} \label{l:step7}
%\begin{enumerate}
%\item[(a)] For a given $y_0$, the $H_0(x|y_0)$ has the unique closed form:
%\begin{equation} \label{eq:iterd} (H_0(x|y_0))_1 = 1/\sqrt{A_{1,1}}, \;\;\mbox{for}\;\; i=2,\dots,p\;\;\; (H_0(x|y_0))_i = \frac{-(y_0)_i+\sqrt{(y_0)_i^2+4A_{i,i}}}{2A_{i,i}}. \end{equation}
%\item[(b)] For a given $y_i$ and $1 \leq i \leq p-1$, the function $H_i(x|z)$ corresponds to the unique solution of a fused lasso problem \citep{tibshirani2011} for the $i$th subdiagonal.
%\end{enumerate}
%\end{lemma}

\begin{lemma} \label{l:step7}
\begin{enumerate}
\item[(a)] For a given $y_0$, $x^*_0$ is unique and its entries have the closed-form:
\begin{equation} \label{eq:iterd} (x^*_0)_1 = 1/\sqrt{(C_0)_{1,1}}, \;\;\mbox{for}\;\; i=2,\dots,p,\;\;\; (x^*_0)_i = \frac{-(y_0)_i+\sqrt{(y_0)_i^2+4(C_0)_{i,i}}}{2(C_0)_{i,i}}. \end{equation}
\item[(b)] For a given $y_i$ ($1 \leq i \leq p-1$),  $x^*_i$ corresponds to the unique solution of the  fused lasso problem \citep[Algorithm 1]{tibshirani2011} for the $i$th subdiagonal of $L$.
\item[(c)]  When  $D$ in (\ref{eq:h1}) is the matrix of
second-order differences, then
\begin{itemize}
\item[(1)] $x_i^*$ corresponds to the solution of the $\ell_1$-trend filtering \citep[Section 6]{kim2009}.
\item[(2)] For $h_i(x|y_i) = 2x^t y_i + x^t C_ix  + \lambda \|Dx\|^2_2,\,(1 \leq i \leq p - 1)$, $x_i^*$ has a closed form and corresponds to the H-P solution:
\[x_i^* = -\frac{1}{2}(C_i + \lambda (D^tD))^{-1}y_i\]
\end{itemize}
\item[(d)] For $\lambda_1 > 0$, the solution of \textit{sparse fused lasso},
\begin{equation} \label{eq:sfused}
\argmin_{x \in \mathcal{R}^{p - i}} \tilde h_i(x|y)= h_i(x|y) +  \lambda_1\|x\|_1,\; 1 \leq i \leq p - 1
\end{equation}
is given by
\[ \hat x_i(\lambda_1,\lambda_2)=\mbox{sign}( \hat x_i(0,\lambda_2))(| \hat x_i(0,\lambda_2)|-\frac{1}{2}(\mbox{diag}(C^{-1}_i))\lambda_1)_{+},\]
where $\hat x_i(0, \lambda_2)$ is the solution of  (\ref{eq:sfused}) when $\lambda_1=0$ and $\lambda_2 \geq 0$.
\end{enumerate}
\end{lemma}

A proof of the lemma is provided in the Appendix. It provides the necessary ingredients for minimizing the objective function (\ref{eq:ssc}) via the following block coordinate
descent  algorithm where each block is a (sub)diagonal of the standard Cholesky factor $L$.
   % We start the algorithm with $(L)^0$ initial Cholesky factor and sample covariance matrix $S$ with positive diagonals elements.

%Using (\ref{eq:H0}), from (\ref{eq:l1}) and (\ref{eq:l11}) we get that
%\[Q_0(L^0|L^\eta)=h_0(L^0|L^{\eta})- (L^{0})^t\sum_{j\neq 0}B_{0j}(L^j)\; \mbox{and} \; Q_i(L^i|L^{-i})=h_i(L^i|L^{-i})-(L^{i})^t\sum_{j\neq i}B_{ij}(L^j).\]  %\;1 \leq i \leq p-i
%Therefore, algorithm which sequentially minimizes (\ref{eq:h0}) and (\ref{eq:h1}) with respect to $x$, for fixed $y_0$ and $y_i$ respectively, will sequentially minimize our objective function $Q(L)$ with respect to diagonal and subdiagonals.%  (\ref{eq:ssc}).

% to minimize $Q_{SSC}(L)$. SSC algorithm is described in Algorithm~\ref{a:SSC}.
\begin{algorithm}[H]
\caption{The SC algorithm}\label{a:SSC}
\begin{algorithmic}[1]
%\Procedure{SSC algorithm}{}
\BState \emph{input}:
\State $\textit{$\epsilon,\lambda,k_{max}$} \gets \textit{Stopping criteria, Tuning Parameter, and max. number of iteration}$
%\State $\textit{$\lambda$} \gets \textit{ Tuning Parameter}$
\State $\textit{${L}^{(0)}$} \gets \textit{Initial Cholesky factor }$
%\State $\textit{$S$} \gets \text{$p \times p$ } \textit{Sample Covarince Matrix}$
\BState \emph{Set $B \gets S\otimes I_p ;\;C_i \gets B_{ii}$} \label{a:SSC:0} %;\; C_i \gets B_{ii};\;k \gets 1
%\State $\textit{$\alpha$} \gets \text{Confidence level for }\textit{conditional independence test}$
\BState \emph{while $\|L^{(k+1)}- L^{(k)}\|_{\infty}>\epsilon$ or $k < k_{max}$}:
\State $\quad L^{(k)} \gets L^{(0)}$
\BState \emph{$\quad $for $i=0,\dots,p-1$ do}:
%\BState \emph{$\quad \quad$if $i==0$}:
%\State  $\quad \quad  y_i = \sum_{j\neq i}B_{ij}L^j$ {\bf Do we need this?}
%\State   $\quad \quad  \hat L^{i} = H_i(L^i|y_i)$  \label{a:SSC:1}  \label{s:hi}      %\min \limits_{L^{0} \in \mathcal{R}^p_+} \Big \{-2\sum_{j=1}^{p-1} logL_{j,j} + (L^{0})^tB_{00}L^{0}+ 2(L^{0})^t\sum_{j\neq 0}B_{0j}(L^j)^{(k)} \Big \} $
\State   $\quad \quad  \hat L^{i} = \argmin h_i(L^i|y_i)$  \label{a:SSC:1}  \label{s:hi}
%\BState \emph{$\quad \quad$else}
%\State  $\quad \quad   z=  \sum_{j \neq i}^{p-1}B_{ij}L^{j}+B_{i0}L^{0}$
%\State   $\quad \quad  \hat L^i= H_i(L^i|z)$ \label{a:SSC:2}  %\min \limits_{L^i \in \mathcal{R}^{p-i}} \Big \{ (L^{i})^tB_{ii}L^{i} + (L^{i})^t(\sum_{j \neq i}^{p-1}B_{ij}(L^{j})^k+B_{i0}(L^{0})^k) + \lambda P(L^i) \Big \}$  %%\|D^{(i)}L^{i}\|_1
\State  \emph{$\quad \quad$Update $L^{(k)}$ by replacing the $i$th subdiagonal by $\hat L^i$ }
\State  $ \quad L^{(0)} \gets L^{(k)};\;k = k+1$
%\State
%\textit{to companion} $\textbf{C } \textit{form}
\BState \emph{Output}:$\;L$
%\State
%\EndProcedure
\end{algorithmic}
\end{algorithm}
We note that the Algorithm~{\ref{a:SSC}} is well-defined so long as the
diagonal entries of sample covariance matrix and the initial Cholesky factor are positive. That is the minimum in the
 optimization appearing in line~\ref{s:hi} of the algorithm is attained. This
follows from Part (b) of Theorem~\ref{l:conv} and the fact that $h_i$'s
are strictly convex functions of $L^i,\; 0\leq i \leq p-1$.

%Next we discuss the convergence properties of Algorithm~\ref{a:SSC}.

\subsection{Convergence of the SC Algorithm}

In this section,  we  establish convergence of the SC algorithm under the weak restriction that the diagonal entries of $S$ are positive.

A key step is to reduce the objective function
(\ref{eq:ssc}) to  the following  widely used objective
 function  in the statistics and machine learning communities \citep{khare2014}:
   %Thus, understanding the theoretical convergence properties of SSC algorithm is crucial.
 %Building on the results in \citet{tseng2001} for the block descent algorithm and  %when penalty function is separable over the blocks and objective function satisfies some appropriate conditions. The author provides that under suitable conditions, every limit point of the sequence of iterates produced by the block coordinatewise minimization algorithm is a stationary point of the objective function. Unfortunately, this does not guaranteed that the sequence of iterates also converges.
% \citet{tseng2009} for a hybrid of gradient and coordinate descent methods, \citet{khare2014} establish
% convergence of a cyclic coordinatewise algorithm for

 %\[
% f_1(x) = g(Ex) + \lambda \sum_{i \in S}|x_i|,
% \] and
\begin{equation} \label{eq:f2}
h(x) = x^t E^tEx - \sum_{i \in C^c}\log x_i + \lambda \sum_{i\in C}|x_i|
\end{equation}
where $\lambda>0$ is a tuning parameter, $C$
is a given subset of indices and the matrix $E$ does
not have a zero column. Since the objective function restricted to  each subdiagonal (line~\ref{a:SSC:1} in Algorithm~\ref{a:SSC}) is strictly convex, a unique global minimum with respect to each subdiagonal is guaranteed even when $n < p$. This additional strict convexity property along with
Theorems 2.1 and 2.2 in \citet{khare2014} are the key ingredients for showing that the iterates in SC algorithm converge to
the global minimum of the objective function $Q$.

\begin{theorem}\label{l:conv}
\begin{enumerate}
\item[(a)] The objective function $Q(L)$ with the fused Lasso penalty admits  the generic form:
\begin{equation} \label{eq:tc}
h(x)= x^tE^tEx - \sum_{i =1}^{p}\log x_i + \lambda \sum_{j \in C}|x_i|,
\end{equation}
where,
$$x=[L_{1,1},\dots,L_{p,p},L_{3,2}-L_{2,1},\dots,L_{p,p-1}-L_{p-1,p-2},\dots,L_{p,2}-L_{p-1,1},L_{p,1}]^t, $$ and the set $C$ of indices consists of the last element of $x$ and along with those  of difference forms, and $E$ is a suitable matrix with no $0$ columns.
%\item[(d)] $Q_{0}(L^0)=-2\sum_{j=1}^{p} logL^0_{j}+ q_0(L^{0})$ is strictly convex functions with respect to $L^{0}$ and for $i=1,\dots,p-1$ and fixed $L^{-{i}}$(as symbolized by the vertical line), strictly convex function $Q_{i}(\cdot)$ is equal $q_i(\cdot)$ with additional penalty term multiplied to positive constant $\lambda$.
% $Q_{i}(L^{i}|L^{{-i}}) = q_{i}(L^{i}|L^{-i}) + \lambda P({L^{i}})$

\item[(b)] If $diag(S)>0$, then the sequence of iterates $\{L^{(k)}\}$ in Algorithm~\ref{a:SSC} converges to a global minimum of $Q$.
%Moreover, $Q(L)$ is coercive, i.e. $Q(L) \rightarrow \infty$ as $\|vec(L)\| \rightarrow \infty$ and an

\end{enumerate}
\end{theorem}

Proof of the theorem given in the Appendix  relies on the following:

\begin{lemma} \label{lemma-np}
 For every $n$ and $p$
$$\inf_{L \in \mathcal{L}_p}Q(L) \geq - \mathbf{1}^t_pK \mathbf{1}_p > -\infty, $$
where $\mathbf{1}_p$ is a $p \times 1$ vector of 1's and $K$ is a positive semi-definite matrix. Moreover, any global minimizer of $Q(L)$ over the open set $\mathcal{L}_p$ lies in $\mathcal{L}_p$.
\end{lemma}

%\begin{remark} We note that there is a slight difference in (\ref{eq:f2}) and (\ref{eq:tc}). In particular, log terms summation in  (\ref{eq:f2})  is over the complement set of $C$, but in our case $x$ has extra elements. However, it can be seen that first, this slight difference does not affect the algorithm and the result. Second, with some extra algebra it can be shown that additional terms in $x$ can be expressed in terms of elements in set $C$ and therefore can be disappear from the equation showing the exact similarity of (\ref{eq:f2}) and (\ref{eq:tc}).
%\end{remark}

A discussion of  convergence of the sequence of iterates for $\ell_1$-trend filtering and HP is provided in the Appendix~\ref{ap:E11}.

%%%%%%%%%%%%%%%%%%%%%%%%%%%%%%%%%%%%%%%%%%%%%%%%%%%%%%%%%%%%%%%%%%

%%%%%%%%%%%%%%%%%%%%%%%%%%%%%%%%%%%%%%%%%%%%%%%%%%%%%%%%%%%%%%%%%%%%%%%%%
 \subsection{Computational Complexity of the SC Algorithm}
 The sequential SC algorithm in each iteration  sweeps over the diagonal and subdiagonals of $L$ where in each sweep it must compute $y_i$ and $h_i$. For example, for fused lasso penalty, from Lemma~\ref{l:l1}, updating each subdiagonal requires solving a fused lasso problem. Therefore, the computational cost of each subdiagonal update depends on the chosen penalty function. Denoting by $R_p$ the computational cost for the chosen penalty to minimize $h_i,\; 1\leq i \leq p-1$, %Here the complexity is shown for $D$ as a second-order difference matrix \citet{kim2009}.
the next lemma provides the computational cost for each iteration of SC algorithm.
\begin{lemma} \label{l:cost}
The computational cost of Algorithm~\ref{a:SSC} in each iteration is $min(O(np^2 + pR_p),O(p^3 + pR_p))$ .
\end{lemma}
The proof is provided in Appendix~\ref{ap:D}. For example, $R_p = O(p)$ for $x^*_i$ for the $\ell_1$-trend filtering penalty \citep{kim2009}. Thus,  the computational cost of the SC algorithm is $min(O(np^2),O(p^3))$ which is comparable to the cost of the existing sequential algorithms such as GLasso \citep{Friedman2008}, SPACE \citep{peng2009} and CONCORD \citep{khare2015} and CSCS \citep{khare2016} when iterations have been run sequentially.

\section{ Connections Among $L$, $T$ and Local Stationarity} \label{s:TandL}

A key feature of our SC algorithm  is its ability to capture the smoothness  of  subdiagonals
of the Cholesky factors through
regularized likelihood estimation rather than the traditional (non)parameteric methods.
 In this section, we explore the connection between smoothness of $T$ and $L$  when the diagonal
  elements of  $\Lambda$ are bounded away from zero.

  Smoothness of  time-varying  covariance and spectral density functions  \citep{dahlhaus1997} and subdiagonals of $L,T$
   are  usually studied    by embedding the underlying nonstationary process in a doubly indexed sequence $X_{t,N}$ (triangular arrays), and
    functions defined on the rescaled time  $u=\frac tN \in [0,1]$. For example,  Figure~\ref{fig:tvar}
    provides a simple illustration of the correspondence between the time-varying AR(1) model in (\ref{ARS}), with
    $N = p - 1$, and  the  subdiagonals of  $T$.

\begin{figure}[th!]
\setlength{\unitlength}{1cm}
\thinlines
\begin{picture}(10,6)
\put(2,2.2){$\begin{aligned}
X_{1,N}& = \alpha_1 \Big(\frac{1}{N} \Big)X_{0,N} + \sigma \Big(\frac{1}{N} \Big)\epsilon_1 \\
%&\vdots \\[-1ex]\\
&\vdots \\[-1ex]
X_{N,N} &= \alpha_1 \Big( \frac{N}{N} \Big)X_{N-1,N} + \sigma \Big( \frac{N}{N} \Big)\epsilon_N \\[-1.5ex]
%&\vdots \\[-1ex]
\end{aligned}$}
%%$X_{1,p} = -\alpha(1/p)X_{t-1,p} + \sigma(1/p)\epsilon_t$}
%\put(2,2.2){$\vdot = \vdot + \vdot $}
%\put(2,2,2){$X_{i,p} = -\alpha(t/p)X_{t-1,p} + \sigma(t/p)\epsilon_t$$}
%\put(2,2.2){$\vdot = \vdot + \vdot $}
\put(9,2.2){ $ T = \begin{bmatrix}
1&0 &  \cdots & 0 \\
- \alpha_1 \Big(\frac{1}{N} \Big) &1& 0  & 0 \\
\vdots& \ddots& \ddots   & \vdots \\
%& & -\alpha_1 \Big( \frac{t}{p} \Big) & \ddots & \ddots & \vdots \\
%& & &\ddots & \ddots & \vdots \\
0& \hdots  &-\alpha_1(\frac{N}{N}) &1
\end{bmatrix}$}
\thinlines
\qbezier(4,3.75)(7,5)(10.5,3.2)
\put(10.5,3.2){\vector(2,-1){0.15}}
\thinlines
\qbezier(4,0.8)(7,0.5)(13.4,1.1)
\put(13.4,1.07){\vector(2,1){0.15}}
\end{picture}
\caption{Depiction of a Time-Varying AR(1) and the Matrix $T$}
\label{fig:tvar}
\end{figure}

The  next lemma  connects  the  smoothness of the entries of the $i$th
 subdiagonal of the Cholesky factors $L, T$ and  the diagonal entries of the matrix $\Lambda$ viewed as functions on $[0,1]$.
       More precisely, the $i$th subdiagonal $L$ and other matrices is viewed as a function of time by writing: $L^i(\cdot):[0,1] \rightarrow R$ where
         $L^{i}(u) = L^{i}(j/N) = L^{i}_{uN}$ stands for its  $j= uN$th element. In this section,  smoothness of a function refers to the function 
         being of bounded total variation where  the total variation (TV) of a function $g(\cdot)$ is defined as $$TV(g)=\sup \Big \{ \sum_{i=1}^{l}|g(x_i)-g(x_{i-1})|:0\leq x_0 < \cdot <x_l \leq 1 \Big \},$$
  for $x_i$'s of the form $\frac iN$.

%The second part uses bounded variation assumption of coefficients and variance of TVAR(p) process  (q=0 in (\ref{eq:tvarma}) ) to establish the connection between subdiagonal smoothness of $T$ and $L$.
\begin{lemma} \label{l:TL}
%For TVAR(p) process and corresponding modified and standard Cholesky factors $L$ and $T$,
(a)
If $\sigma(i)>c>0$, then for any $u,v \in [0,1]$ of the form $t/N$, we have
 $$| L^i(u)-L^i(v)| \leq c^{-1}|T^i(u)-T^i(v)| + c^{-2} |T^i(u)||\sigma (u)-\sigma (v)|$$

 %for any $u,v \in [0,1]$ of the forms $\frac {t}{T}$ and $\frac {s}{T}$ we have

% $$| L^i(u)-L^i(v)| \leq c^{-1}|T^i(u)-T^i(v)| + c^{-2} |T^i(u)||\sigma (u)-\sigma (v)|,$$
 %where the $t= uT$th element $L^{i}_{uT} = L^{i}(u) = L^{i}(t/T)$ highlights the fact that the $i$th subdiagonal is viewed as a function of time.

%\item [b)]  As well, for TVAR(p) process if coefficients $\alpha_j(\cdot)$ and $\sigma(\cdot)$ are of bounded variation then for all $\delta > 0$ there exist a finite partition of intervals $I_1 \cup \dots \cup I_m = [0,1]$ such that
%\[ ||\Delta_{uv}(L^i)|  - \frac{1}{\sigma(v)} |\Delta_{uv}(T^i)|| \leq \delta,\]
%where $u,v$ belong to the same $I_k$.
%\end{enumerate}
(b) If in addition,  $\sigma(\cdot)$ and the $i$th subdiagonal $T^i(\cdot)$ are functions of bounded total variation on the rescaled interval $[0,1]$ with $TV(T^i) \leq K_1$, $TV(\sigma) \leq K_2$, and $\|T^i\|_{\infty} < m$ ,  then $L^i$ is of bounded total variation and
 \begin{equation} \label{eq:tv}
 TV(L^i) \leq c^{-1}K_1 + c^{-2}K_2m.
 \end{equation}
\end{lemma}
The proof of the lemma is provided in the Appendix~\ref{ap:E1}.

The requirement of being of bounded variation on  $\sigma$ and $T^i$ open up a
  window to connect and extend the class of time-varying AR models to locally stationary processes.
  In particular, a process $X_{t,N}$ (t=1,\dots,N) with a time-varying MA($\infty$)-representation:
          is  locally stationary
        \citep{dahlhaus1997,dahlhaus2009, dahlhaus2012}
          if
\begin{equation} \label{eq:ma}
X_{t,N}=\sum_{j=-\infty}^{\infty}a_{t,N}\epsilon_{t-j},
\end{equation}
where  $a_{t,N}$'s are such that there exists functions $a(\cdot,j):[0,1)\rightarrow R$ satisfying $$TV(a(\cdot,j)) \leq \frac{K}{l(j)}\; \mbox{and }\;\mbox{sup}_{j}\sum_{t=1}^N |a_{t,N}(j)-a(\frac{t}{N},j)|\leq K,$$
for a constant $K$  and
$l(j) = 1,\;\mbox{for}\,|j| \leq 1$ and $l(j) = |j|(\log|j|)^{1 + k}$ otherwise.
% the total variation (TV) of a function $g(\cdot)$ is defined as $$TV(g)=\sup \Big \{ \sum_{i=1}^{l}|g(x_i)-g(x_{i-1})|:0\leq x_0 < \cdot <x_l \leq 1 \Big \}.$$
It follows from (\ref{eq:ma}) that the time-varying spectral density and the lag-$k$ covariance  at  the rescaled time $u=t/N$ are of the form
\[
\begin{aligned}
f(u, \lambda) &= \frac{1}{2 \pi} |A(u,\lambda|^2 \\
c(u,k) & = \int_{-\pi}^{\pi} f(u, \lambda) exp(i \lambda k) d\lambda = \sum_{j = -\infty}^{\infty} a(u, k+j) a(u,j),
\end{aligned}
\]
where
\[A(u, \lambda) = \sum_{j = -\infty}^{\infty}a(u,j)exp(-i\lambda j).\]

The  time-varying AR models in (\ref{ARS}) can be enlarged to the class of locally stationary
time-varying ARMA  models \citep[Proposition 2.4]{dahlhaus2009} by choosing its coefficients
and the variance   functions to be of bounded variation. More precisely, if all the coefficients $\alpha_j(\cdot), \beta_k(\cdot)$, and
 the variance functions $\sigma^2(\cdot)$ are of bounded variation, then under usual conditions on the roots of the characteristic polynomials, the system of difference equations
\begin{equation} \label{eq:tvarma}
\sum_{j =0}^{p} \alpha_j \Big( \frac{t}{N} \Big) X_{t-j,p} = \sum_{k=0}^{q} \beta_k \Big( \frac{t}{N} \Big) \sigma \Big( \frac{t-k}{N} \Big) \epsilon_{t-k}
\end{equation}
has a locally stationary solution of the form (\ref{eq:ma}).

A related  topic of interest is the connections
between  smoothness of the standard Cholesky factor and the covariance matrix of nonstationary processes.
Interestingly, it is known \citep[Lemma 2.8]{chern2000} that  smoothness of a covariance (positive-definite matrix-valued)
function is inherited by its unique standard Cholesky factor when smoothness  is in terms of degree of differentiability.
Furthermore, it is known  \citep[Proposition 5.4]{dahlhaus2009} that the subdiagonals of
 covariance matrices of locally stationary processes are functions of bounded variation. Next, we establish  the connection
 (equivalence) between the subdiagonals  of
 the standard Cholesky factor and a covariance matrix being  of bounded variation.

\begin{prop} \label{p:p1}
\begin{itemize}
\item[(a)] If the (sub)diagonals of the Cholesky factor $L$  are of bounded  variation on the rescaled interval $[0,1]$ with $TV(L^i) \leq K_i,\; \|L^i\|_{\infty} \leq m_i\; (0 \leq i \leq p -1) $, then the (sub)diagonals  of the matrix $\Sigma = L^tL$ are of bounded  variation with
\[
TV(\Sigma^i) \leq \sum_{j = 0}^{p - i -1} (m_jK_{j+i} + m_{j+i}K_j)
\]
\item(b) The converse of (a) is true.
\end{itemize}
\end{prop}
The proof is relegated to the Appendix.

%In our settings, to infer about the structure of $T$ from $L$, we need that the neighbors of each subdiagonal element to have similar corresponding diagonal values in the sense of (\ref{eq:e1}), which holds under a local stationarity assumption. This property is well expressed when $p$ get larger and the grid in each time interval gets finer. Therefore, in this context, the assumption of slowly growing/decreasing variance is justified. Thus, if two consequent elements of diagonal $|\sqrt {\Lambda_{ii}}- \sqrt{\Lambda_{i+1i+1}}|$ is small then from (\ref{eq:e1}), if the subdiagonal elements of $T$ have smooth variation, then the structure is transferred to subdiagonals of $L$.

%the following two step procedure to reconstruct $\hat T$ after estimating $\hat L$ is justified.
  %First, after standardization of the data, the estimated diagonal elements of $\hat \Lambda$ smoothly oscillate around 1 thus forcing the subdiagonals of $L$ to mimic (upto negative constant) the smooth structure of subdiagonals of $T$.
 %Next, to recover $\hat T$, we suggest the following element-wise multiplication
 %Our simulations  show that for the most cases the first step is enough for recovering smooth structure of $T$ when $p$ is comparably large . However, as discussed in \citet{khare2016}, there can be cases, in particular, existence of latent variables, when standardization is not enough.

%%%%%%%%%%%%%%%%%%%%%%%%%%%%%%%%%%%%%%%%%%%%%%%%%%%%%%%%%%%%%%%%
\section{Simulation and Data Analysis} \label{s:s3}

%\subsection{Simulation Study}\label{s:s3}

%%%%%%%%%%%%%%%%%%%%%%%%%%%%%%%%%%%%%%%%%%%%%%%%%%%%%%%%%%%%%%%%%%%%%%%

In this section, we illustrate and gauge the performance of our methodology using simulated and real datasets.
 We use three commonly used penalty functions: fused lasso, $\ell_1$-trend filtering and
  Hodrick-Prescott (H-P) filtering \citep{hodrick1997}, and the corresponding SC algorithm is referred to
   as SC-Fused, SC-Trend and  SC- HP, respectively.

\subsection{The Simulation Setup: Four Cases of T}

 In all simulations, the sample sizes are $n=50, 100$, and dimensions $p=50, 150$, covering settings
where $p<n$ and $p>n$, respectively. Each simulated dataset is centered to zero and scaled to unit variance.  The tuning parameter $\lambda$ is chosen from the range $[0.1,1]$ over $100$ equally spaced grid points using the
BIC and CV criterion described in the Appendix~\ref{s:tun}. We repeat the simulation 20 times.
As inputs to the Algorithm~\ref{a:SSC}, we set the tolerance
 $\epsilon = 10^{-4}$ and the initial Cholesky factor is the diagonal matrix with diagonal elements equal to $\sqrt{diag(S)}$.

 We start with a pair $(\Lambda, T)$ and use the parameterization $L=\Lambda^{-1/2}T$  as in \citep{khare2016} where $\Lambda$ is a diagonal matrix and $T$ is a unit
lower-triangular matrix constructed for  the four cases A-D described below.
For given pairs $(n,p), (T,\Lambda)$, sample data are
drawn independently from $N_p(0,(L^tL)^{-1})$. %We consider $p=50,150$, which corresponds to
In each case, except for the Case B, where the number of nonzero subdiagonals is equal 2,  the number of non-zero
subdiagonals is restricted to be $5$, that is in each iteration the SC algorithm sweeps only over the first 5 subdiagonals and the rest of subdiagonals are set  to 0.  Except for the Cases A and B,  construction of the matrix $T$ starts with generating its first subdiagonal, and then filling
the rest of its subdiagonals by  eliminating the last element of the previous subdiagonal. The diagonal elements of $\Lambda^{1/2}$, for the Cases A and B are equal one and are of the form $\log((1:p)/10+2)$ for the Cases C and D.

  The four cases of $T$ with varying degrees of smoothness
(nonstationarity) of their subdiagonals and the diagonal matrix $\Lambda$ considered are:
%For the sake of simplicity and taking into consideration our main interest in this section, in Algorithm~\ref{a:SSC}, we force all subdiagonal elements that are greater then $k$ to be 0. In other words, Algorithm~\ref{a:SSC} alternates between diagonal and $k$ subdiagonals. Each $i'$th subdiagonal takes the following form for $i =1,\dots k$.
\begin{itemize}
%\item[Case A:] We generate subdiagonal as a piecewise linear function, such that each piece has different intercept and slope. In particular, the first subdiagonal constructed based on $3$ piecewise linear functions, such that $L^1_t=\alpha_j + \beta_j t/p + z_t$ for $j=1,2,3$ . Here $\alpha_j$ and $\beta_j$ are taken from uniform $[-1,1]$ distribution and $z_t \sim N(0,0.7^2)$.
%\item[Case B:] This case is similar to the previous one, but here we consider only two linear functions and force the kink points to be joined.
\item [Case A:] A stationary AR(1) model where $T$ is a Toeplitz matrix with the  value for the first subdiagonal randomly chosen from the uniform
 distribution on $[0.3, 0.7]$.% {\bf Say more about the first subdiagonals}
%and the constant element in the first subdiagonal is randomly chosen from the uniform [0.3,0.7] distribution.
\item[Case B:] Resembles an AR(2) model as in \citet[Section 4,1]{davis2006} dealing with piecewise stationary processes:
\[X_t = \begin{cases} -0.7X_{t-1}+ \epsilon_t& 1 \leq t \leq p/2 \\ 0.4X_{t-1} - 0.81 X_{t-2} + \epsilon_t & p/2 < t \leq 3p/4 \\-0.3 X_{t-1} - 0.81X_{t-2} + \epsilon_t & 3p/4 <t \leq p \end{cases},\]
where $\epsilon_t \sim N(0,1)$. The matrix $T$  here is 2-banded and the diagonal elements of $\Lambda^{1/2}$ are equal to 1 (See Figure~\ref{fig:rs1}).
\item[Case C:] The first subdiagonal of $T$ is given by $T^1_i= 2 (i/p)^2- 0.5$, $i=1,\dots,p-1$, corresponding
to a (time) varying-coefficient AR model \citep{Wu2003}.
\item[Case D:] The first subdiagonal of $T$ is generated according to
\[T^1_i= x_i +z_i,\; i= 1,\dots,p-1,\;\; x_{i+1}=x_i + v_i,\;i= 1,\dots,p-2,\]
with $x_1 =0,z_i \sim N(0, 1)$ and $v_t$ is  a simple Markov process \citep[Section 4]{kim2009}. That is with probability m, $v_{i+1}=v_i$ and with probability $1-m$ it is chosen from the uniform distribution $[-b,b]$ where $m=0.8, b=0.5.$
\end{itemize}
%For all four cases the diagonal elements are given by $L^0=log((1:p)/10+2)$.
%For comparison wef include two other, CSCS and HSC sparse precision matrix estimators designed for the ordered-variable case.
 Figure~\ref{fig:rs1} illustrates plots of the first subdiagonal of the matrix $T$ versus the rescaled time in $[0,1]$ for the four cases  with $p=50$.

\begin{figure}[htb!]
\centering
\includegraphics[width=14cm,height=4.2cm]{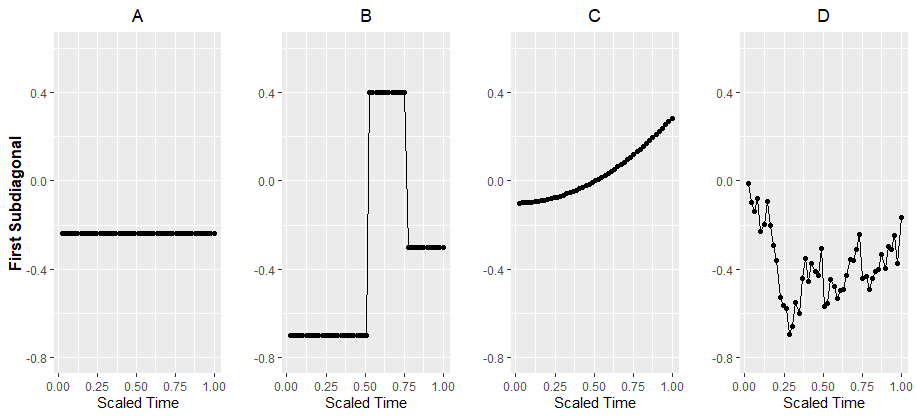}
\caption{ Cases A-D, plots of the first subdiagonal of $T$  vs rescaled time ($p=50$).}
\label{fig:rs1}
\end{figure}

\subsection{Capturing Smoothness: A Graphical Comparison}
First, we assess graphically the ability of our methodology to learn the varying degrees of smoothness of the first subdiagonal for
 the four cases introduced above.
Figures~\ref{fig:rs2}
and \ref{fig:rs3} illustrate the simulation results using the SC algorithm for $p=50$ and $150$, respectively. In each 2 by 4
 layout, each column corresponds to one of the four cases and the row to the criteria (BIC or CV) for choosing the tuning parameters.
The results
 for $n = 50$ and $n = 100$ were similar, therefore we report only those  for the larger sample size.
%Table~\ref{t:t1} reports the selected tuning parameter using BIC and the cross-validation criteria, respectively . The
% tuning parameters from Table~\ref{t:t1} have been used to estimate the first subdiagonal of
% matrix $\hat L$ using the SSC algorithm for three different penalty functions.

%The line
%colors blue, red and green correspond to H-P, 1d fused
% lasso, and $\ell_1$ trend filtering penalization forms, respectively.

\begin{figure}[htb!]
\centering
\includegraphics[width=14.2cm,height=8.2cm]{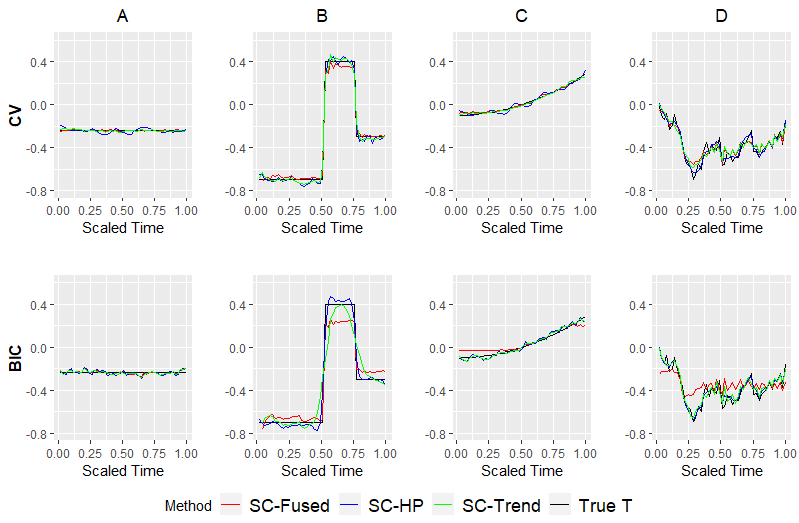}
\caption{Estimated first subdiagonal of $T$ for SC-HP, SC-Fused and SC-Trend ($p=50$).}
\label{fig:rs2}
\end{figure}

\begin{figure}[htb!]
\centering
\includegraphics[width=14.2cm,height=8.2cm]{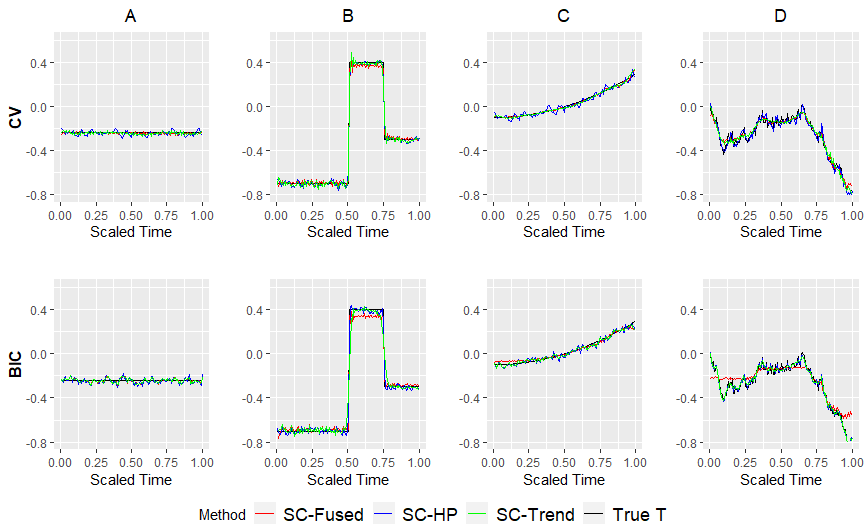}
\caption{Estimated first subdiagonal of $T$ for SC-HP, SC-Fused and SC-Trend ($p=150$).}
\label{fig:rs3}
\end{figure}

The simulation results in both figures provide ample evidence on the good performance  of the SC method for
 estimating time-varying subdiagonals. In particular, for the Case A, as expected, the SC-Fused  learns perfectly the flatness (stationarity) of the first subdiagonal, showing only some wiggliness for the BIC. For the Case B, which corresponds to  a piecewise stationary process, estimators tuned using CV and BIC correctly identify the jumps and show small oscillation around the flat segments. The CV criterion shows an advantage over the BIC for the Case C. More specifically, the SC-Trend  learns better the quadratic structure of the first subdiagonal than the other estimators. For the case D the SC-Trend and SC-HP  provide nearly identical estimates of the first subdiagonal. The results for the other subdiagonals nearly match those in Figures~\ref{fig:rs2} and \ref{fig:rs3}, and are omitted. As $p$ gets larger, there seems to be evidence of improvement in performance of the SC algorithm.

\subsection{Comparing  Estimation Accuracies} \label{s:MoP}

In this section, we compare  the accuracies of the three SC estimators: SC-HP, SC-Fused and  SC-Trend. The overall measures of
performance involve
magnitudes of the estimation errors $\hat T-T$ and $\hat L - L$,  as measured by the scaled Frobenius norm $\frac{1}{p}\|\hat A - A\|^2_F$, and
 the matrix infinity norm $|\|\hat A - A|\|_{\infty}$  for a $p \times p$ matrix $A$.

 %quantifying the smoothness of the estimated subdiagonals of $T$
%as a function of time, namely going beyond the graphical comparisons of the previous subsection,then we employ a new multiscale cross-sample entropy (MCSE)
%measure \citep{xia2012, richman2000} where its smaller values are preferred.
% More precisely, the
%MSCE compares each subdigonal of the matrix $T$ with its estimate over the set of scales as
% a function of time. For example, the scale factor 1 corresponds to the original series, the scale factor 2 to the new series
%  constructed by summing 2 neighboring observations and so on.

Boxplots of the overall estimation  errors  for the matrix $T$ are reported in Figures \ref{fig:aplot} through \ref{fig:dplot}, where
 each figure corresponds to a particular case, each row  to a value of $p$ and the two columns
  correspond to   using BIC and CV criteria, respectively. They corroborate   the findings in the graphical
   explorations Figures~\ref{fig:rs2} and \ref{fig:rs3}, in that the SC-Fused shows tendency to capture well cases
    with constant subdiagonals, SC-Trend and HP are better in  capturing the wiggliness and smoothness of the subdiagonal. The
    corresponding estimation errors for the matrix $L$ show similar patterns, and are thus omitted.
% the As expected for case A, SC-Fused does better for all $p$ and error measures. For case B, which corresponds to piecewise stationary case, for most error measures and $p$, again SC-Fused shows better performance. The case C, does not show clear cut between estimators through the error measures and $p$. The performance of estimators vary over the error measures.

\begin{figure} [H]
\centering
\includegraphics[width=1\linewidth, height = 8.2cm]{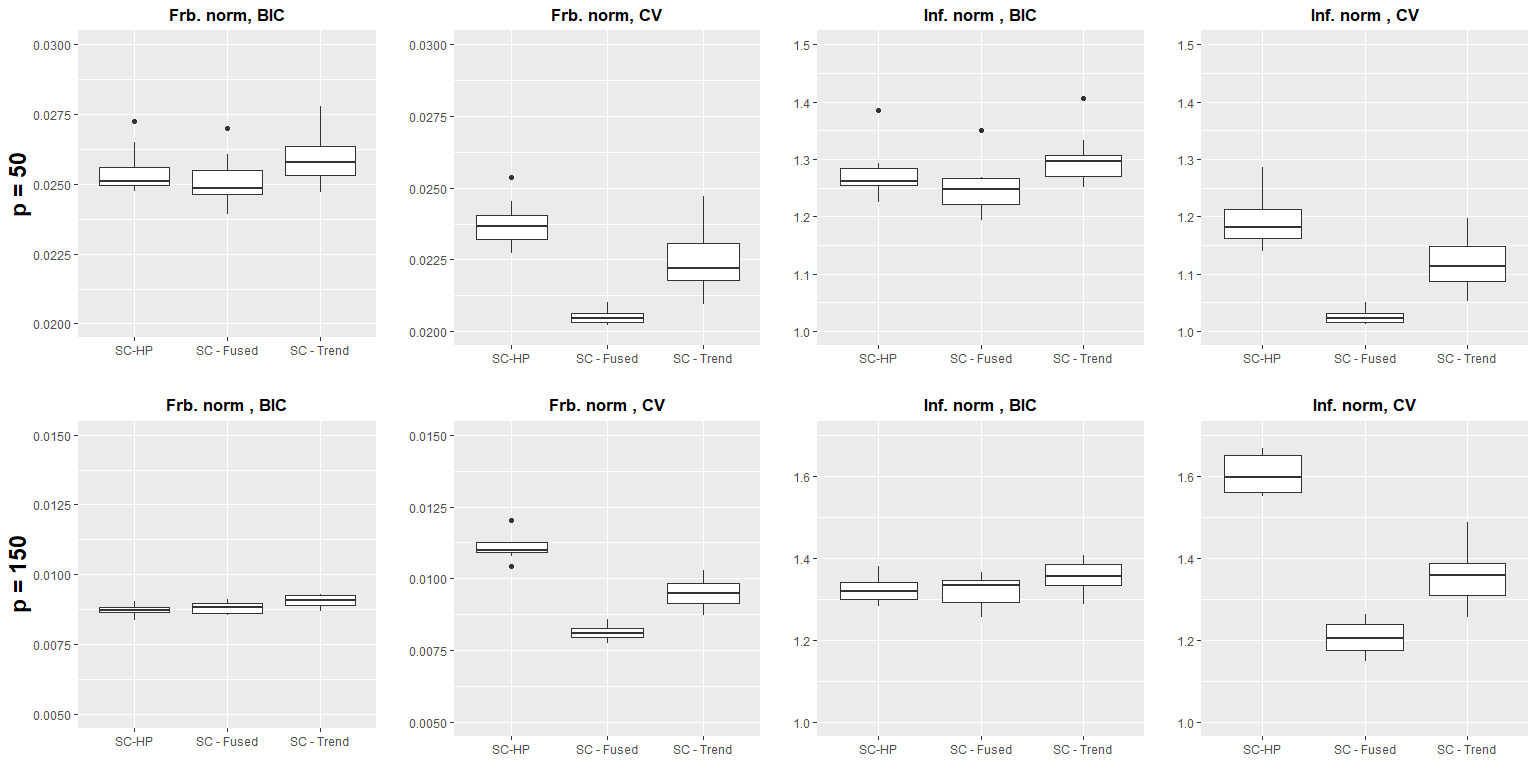}
\caption{Estimation accuracy when data are generated from Case A.}
\label{fig:aplot}
\end{figure}

\begin{figure}[H]
\centering
\includegraphics[width=1\linewidth, height = 8.2cm]{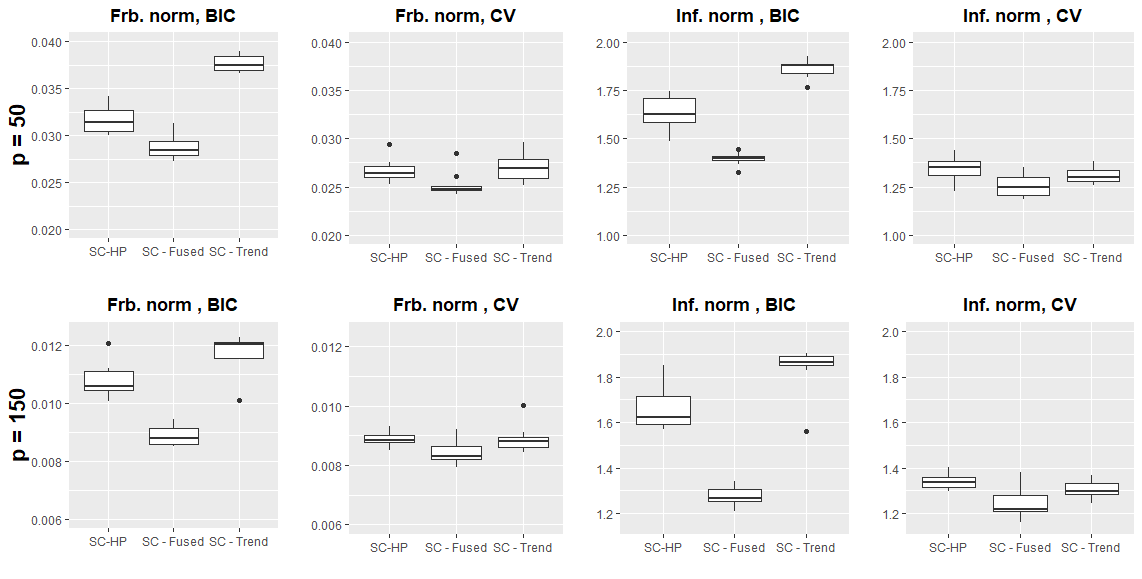}
\caption{Estimation accuracy when data are generated from Case B.}
\label{fig:bplot}
\end{figure}

\begin{figure}[H]
\centering
\includegraphics[width=1\linewidth, height = 8.2cm]{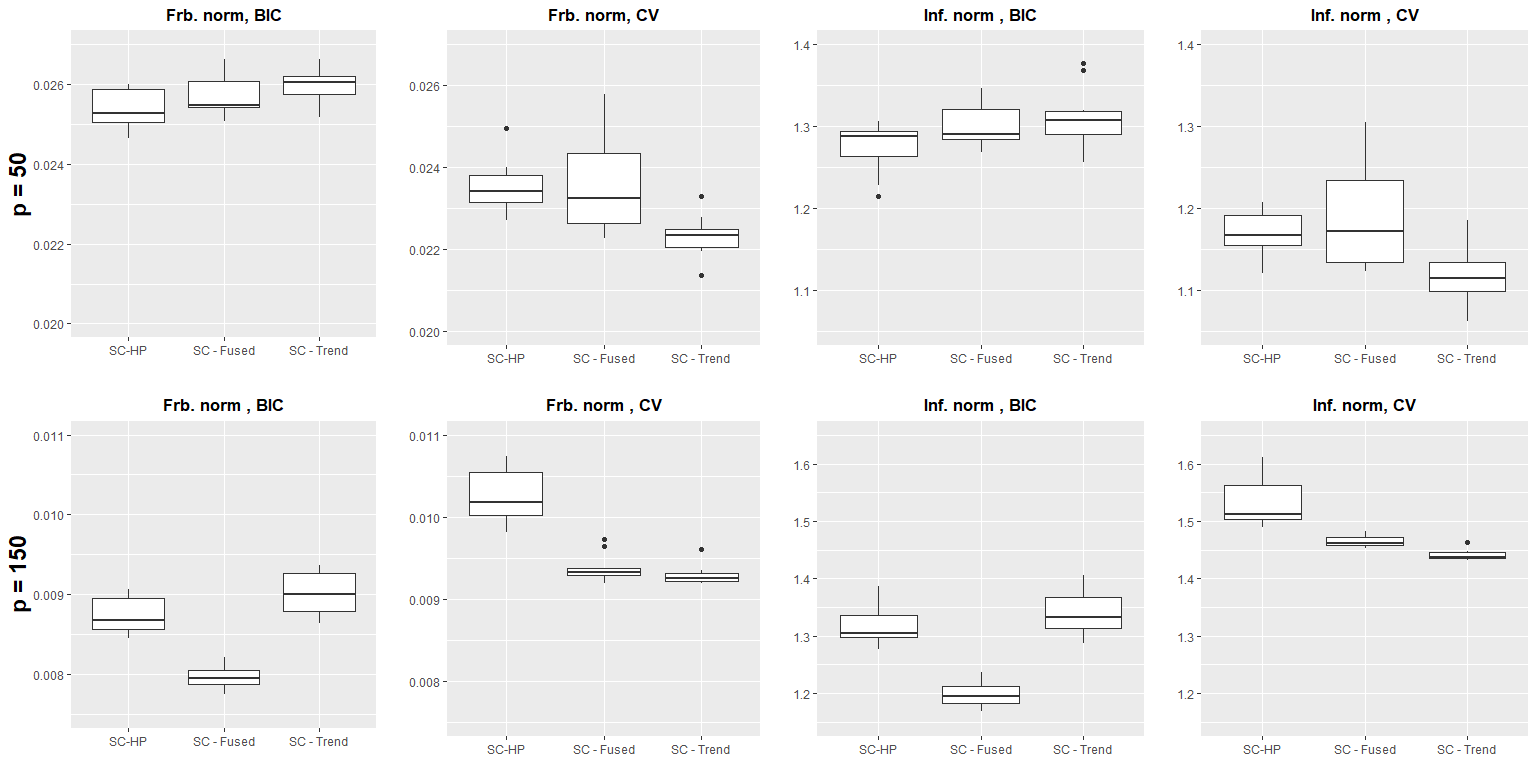}
\caption{Estimation accuracy when data are generated from Case C.}
\label{fig:cplot}
\end{figure}

\begin{figure}[H]
\centering
\includegraphics[width=1\linewidth, height = 8.2cm]{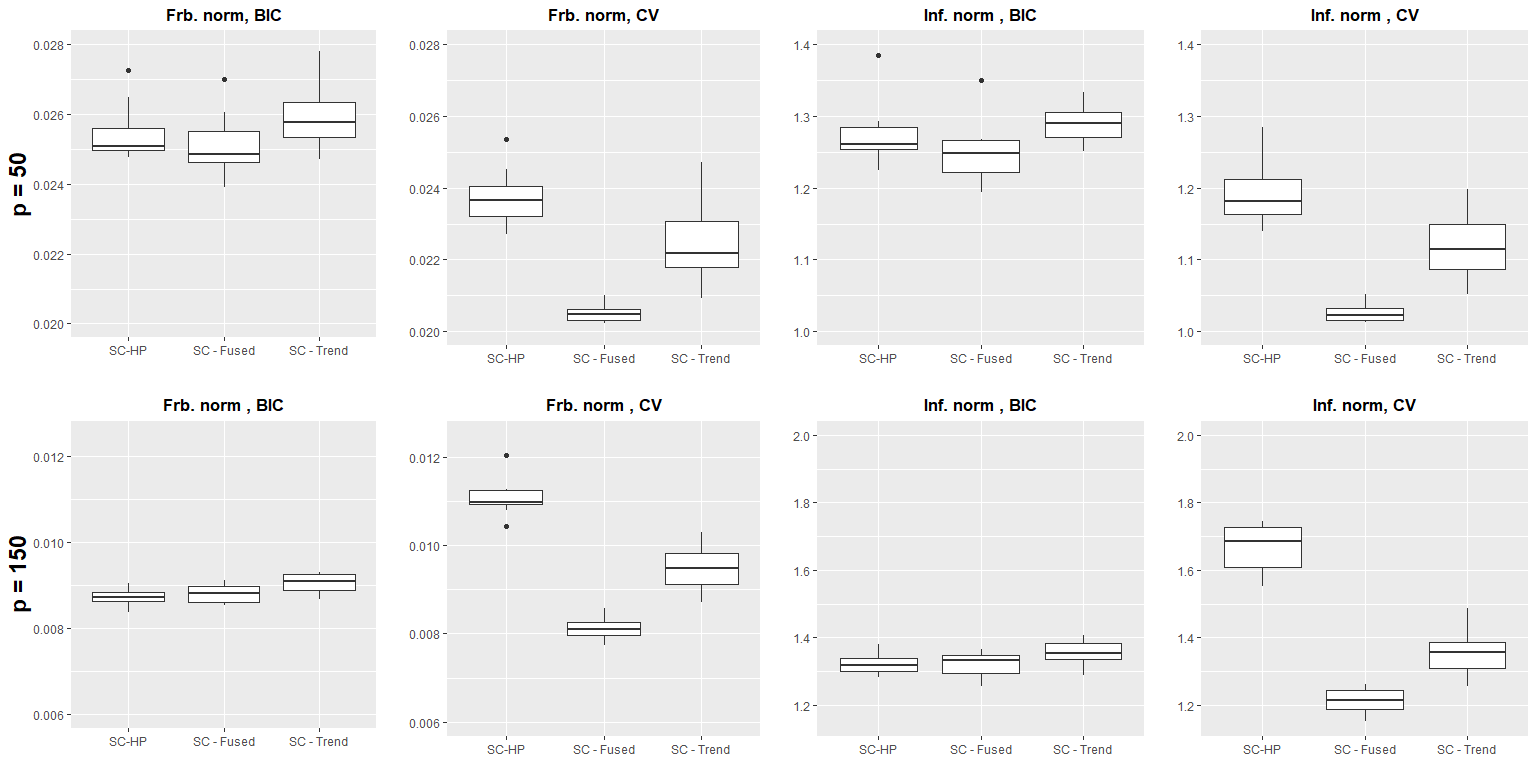}
\caption{Estimation accuracy when data are generated from Case D.}
\label{fig:dplot}
\end{figure}

  In the Appendix~\ref{s:adsim} we provide two additional simulations for a more general matrix $T$ ($L$) than
   those in Cases A-D, and to compare our sparse SC with the existing sparse Cholesky estimators (CSCS, HSC) so far support recovery is concerned. The  results confirm the good performance of the SC method. The two general matrices  are: (1) $T$ is a full lower triangular matrix and its subdiagonals are chosen randomly from the Cases (A-D), (2)
     $T$ has a nonhierarchical structure \citep{yu2017}, that is nonzero subdiagonals are followed by block zero subdiagonals and again by nonzero subdiagonals. %See Figure~\ref{fig:nhsnap} for an illustration.
\subsection{Covariance  Estimators}
 In this section, we assess the performance of our method on learning (inverse) covariance matrices for the Cases A-D.
  We compare our SC method (Fused, HP, Trend) with the CSCS and HSC methods. To make them comparable, instead of limiting the SC algorithm to run over the first five subdiagonals, as in the last two sections, here we use  the more general sparse SC estimator
   (see Lemma \ref{l:step7}) with the two  tuning parameters $\lambda_1$ and $\lambda_2$, respectively. Due to space limitation, we report
    results only for the $p = 150$ with the tuning parameters  selected using the CV criterion.

 We evaluate performance of the estimators using the scaled Kullback-Leibler loss $\frac{1}{p}\Big [tr(\hat \Omega \Sigma) - \ln |\hat \Omega \Sigma| - p \Big ]$ for the inverse covariance and scaled Frobenious norm for the covariance matrix. From results reported in Figures~\ref{fig:frbsigma} and \ref{fig:klomega} for cases A, B, and C, it is evident that
 the SC algorithm learns the covariance matrix better than the SCSC and HSC methods. In particular, for the case A, SC-Fused provides the lowest error measure and for cases B and C, SC-Trend and HP are the lowest. For the Case D, the  HSC is the best.
 For learning the inverse covariance matrix, the SC performs better for all the four cases.

\begin{figure}[H]
\centering
\begin{subfigure}{.5\textwidth}
\centering
\includegraphics[width=1\linewidth, height = 6.5cm]{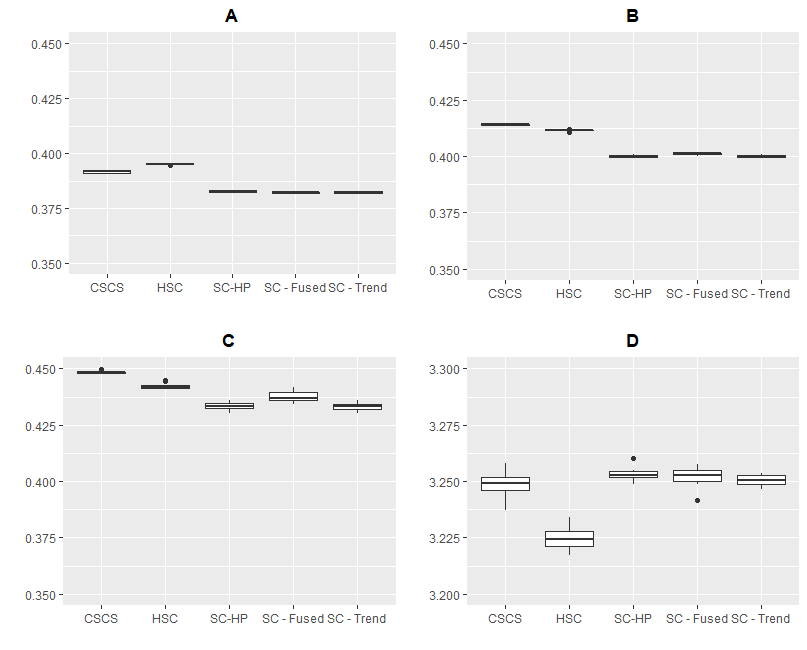}
\caption{Frobenious norm}
\label{fig:frbsigma}
\end{subfigure}%
\begin{subfigure}{.5\textwidth}
\centering
\includegraphics[width=1\linewidth, height = 6.5cm]{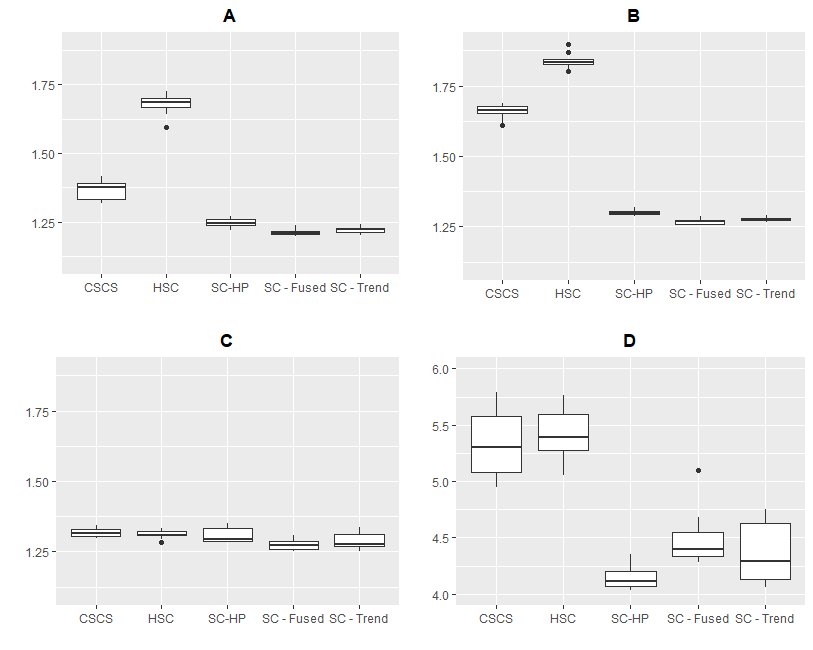}
\caption{Kullback - Leibler loss}
\label{fig:klomega}
\end{subfigure}
\caption{Performance of covariance and inverse covariance matrix estimators for $p = 150$.}
\label{fig:covest}
\end{figure}

%%%%%%%%%%%%%%%%%%%%%%%%%%%%%%%%%%%%%%%%%%%%

\subsection{The Cattle Data} \label{s:cattle}
This dataset  \citet{kenward1987} is from an experiment in which cattle were assigned randomly to two treatment groups A and B. The weights of animals were recorded to study the effect of treatments on intestinal parasites. The animals were weighed $p=11$ times over 122 days. Of 60 cattle $n=30$ received treatment $A$ and the other $30$ received treatment $B$. The dataset has been widely used in the literature of longitudinal data analysis \citep{Wu2003};\citet{huang2007}.

The classical likelihood ratio test rejected equality of the two within-group covariance matrices, thus it is recommended to study each treatment group's covariance matrix separately. In this paper, we report our results for the group A cattle.
It is known \citep{zimmerman2010} that the variances and the same-lag correlations are not constant, but  tend to increase over time
, so that the covariance exhibits  nonstationarity features.
To learn the $11 \times 11$ covariance matrix, we apply the following methods : SC (HP, Fused, Trend), sample covariance S, unstructured antedependence (AD)  \citep[Section 2.1]{zimmerman2010}, autoregression process (AR), variable-order antedependece (VAD) \citep[Section 2.6]{zimmerman2010} and the structured AD model in \citep{pourahmadi1999}, referred to  as POU in the following plot. More specifically,  following \citet[Section 8.2]{zimmerman2010} we consider AD(2), VAD(0,1,1,1,1,1,1,2,2,1,1), AR(2), and POU model for which the log-innovation variances are a cubic function of time and the autoregressive coefficients are a cubic function of lag.  Tuning parameters for all three SC methods were selected using a $5-$fold cross-validation.

% Here we utilize sparse SC method described in Appendix~\ref{ap:D1}, that is we select two tuning parameters $\lambda_1,\lambda_2$, which control sparsity and the smoothness, respectively. Tuning parameters were selected using $5-$ fold cross-validation.

\begin{figure}[htb!]
\centering
\includegraphics[width=14cm,height=6cm]{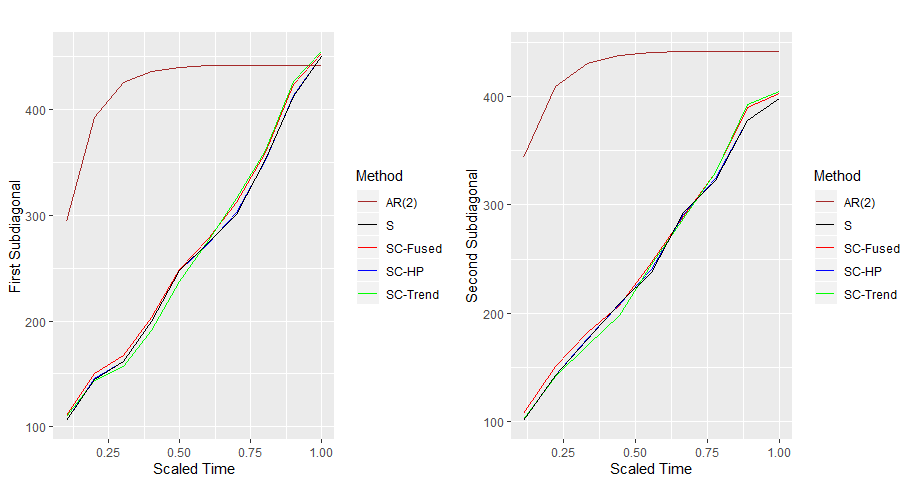}
\caption{Plots of estimated first and second subdiagonals of covariance matrix for the  various estimation methods.}
\label{fig:cplot}
\end{figure}

            We plot  the first two subdiagonals of estimated covariance matrices  for
            the SC(HP, Fused, Trend), S and AR(2) methods in Figure~\ref{fig:cplot}. It
             can be seen that the estimators of subdiagonals provided by the SC methods are almost identical to those of
              the sample covariance matrix. However, the estimated subdiagonals from AR(2) illustrate a different behavior suggesting
              that the data does not support the underlying AR model.
In Appendix we provide similar plots for all eight estimators.  In Table~\ref{tb:likcattle} we report the values of the
negative log-likelihood  for various methods which also
 confirm the results in Figure~\ref{fig:cplot}. The maximum log-likelihood is in bold.

\begin{table}[!htbp] \centering
  \caption{Log-likelihood values for various estimation methods.}
  \label{tb:likcattle}
\begin{tabular}{@{\extracolsep{5pt}} lc}
\\[-1.8ex]\hline
Method &  \\
\hline \\[-1.8ex]
SC-HP & $$-$541.836$ \\
SC-Fused & $$-$547.129$ \\
SC-Trend& $$-$546.573$ \\
AD(2) & $$-$541.451$ \\
VAD & $$-$542.861$ \\
AR(2) & $$-$1,637.894$ \\
POU & -$862.430$ \\
S & $-\mathbf{529.4207}$ \\
\hline \\[-1.8ex]
\end{tabular}
\end{table}

\subsection{The Call Center Data}
In this section, we assess the forecast performance of the SC, CSCS, and HSC algorithms by analyzing the call center data \citep{Huang2006},  from a
 call center in a major U.S. northeastern financial organization. For each day in 2002  phone calls were recorded from 7:00 AM
 until midnight, the 17-hour interval was divided into 102 10-minute subintervals, and the number of calls arrived at the service queue during each interval were counted. Here, we focus on weekdays only, since the arrival patterns on weekdays and weekends differ.

   We denote the counts  for day $i$ by the vector $N_i=(N_{i,1},\dots,N_{i,102})^t,\; i =1,\dots,239$, where $N_{i,t}$ is the number of calls arriving at the call center for the $t$th 10-minute interval on day $i$. The square root transformation
   $x_{it}=\sqrt{N_{it}+1/4},\; i=1,\dots,239,\;t=1,\dots,102$,  is expected to make the distribution closer to normal.
The estimation and forecast performances are assessed by splitting the 239 days into training and test datasets. In particular, to
 estimate the mean vector and the covariance matrix, we form the training dataset from the first $T$ days
  ($T = 205,150,100,75$). Six covariance estimators, five penalized likelihood methods, SC (HP, Fused, and Trend), CSCS and HSC, along with
  $S$ were used to estimate the $102 \times 102 $ covariance matrix of the data. The tuning (penalty) parameters were selected using 5-fold cross validation described in Section~\ref{s:tun}. We report the log-likelihood \citep{khare2016} for the test
   dataset evaluated at all above estimators in Table ~\ref{fg:loglik}, where the largest  value in each column is
   in bold. For all training data sizes, the SC algorithm demonstrates superior performance compared to the other methods. In particular, for $T = 205, 150$, the SC-Trend is the best, but for $T = 100,75$ the SC-Fused provides better results.

\begin{table}[!htbp] \centering
  \caption{Test data log-likelihood values for various estimation methods with training data size 205,150, 100, 75.}
  \label{fg:loglik}
\begin{tabular}{@{\extracolsep{5pt}} cccccc} \\
\hline \\[-1.8ex]
Methods& & \multicolumn{4}{c}{Training data size} \\
\cline{3-6} \\[-1.8ex]
& & 205 & 150 & 100 & 75 \\
\hline \\[-1.8ex]
\multirow{3}{*}{SC}
&HP & $$-$14,435.700$ & $$-$9,018.556$ & $$-$7,472.817$ & $$-$7,467.412$ \\
&Fused & $$-$13,123.300$ & $$-$8,587.785$ & $\mathbf{-7,034.868}$ & $\mathbf{-7,097.938}$ \\
&Trend & $\mathbf{-12,274.970}$ & $\mathbf{-8,477.271}$ & $$-$7,040.924$ & $$-$7,222.989$ \\
\hline \\[-1.8ex]
\multirow{2}{*}{Sparse Cholesky}
&CSCS & $$-$16,814.450$ & $$-$9,754.996$ & $$-$7,484.153$ & $$-$7,365.298$ \\
&HSC & $$-$14,382.330$ & $$-$8,971.729$ & $$-$7,395.206$ & $$-$7,342.343$ \\
\hline \\[-1.8ex]
\end{tabular}
\end{table}

%Table~\ref{t:ccse} reports the SE for the first and second subdigonals of the estimated covariance matrices obtained from the above six estimators, respectively. The minimum value for each subdiagonal have been marked bold. For the first subdiagonal, fused lasso penalty provides the lowest estimate and for the second subdiagonal fused and HSC share the minimum value.
%\begin{table}[!htbp] \centering
%\caption{Sample Entropy for the first and second subdiagonals.}
%\label{t:ccse}
%\begin{tabular}{@{\extracolsep{5pt}} ccc}
%\\[-1.8ex]\hline
%\hline \\[-1.8ex]
%& First S-diag & Second S-diag. \\
%\hline \\[-1.8ex]
%Fused & \textbf{0.03} & \textbf{0.02} \\
%HP & 0.35 & 0.24 \\
%$\ell_1$ trend & 0.07 & 0.03 \\
%CSCS & 0.95 & 0.95 \\
%HSC & 0.22 & \textbf{0.02} \\
%S & 0.76 & 0.80 \\
%\hline \\[-1.8ex]
%\end{tabular}
%\end{table}

  Next, we focus on forecasting the number of call arrivals in the later half of the day using arrival patterns in the earlier half of the day \citep{Huang2006}. In particular, for a random vector $\mathbf{x}_i=(x_{i,1},\dots,x_{i,102})^t$, we  partition $\mathbf{x}_i=((x^{(1)}_i)^t,(x^{(2)}_i)^t)^t$ where $x^{(1)}_i$ and $x^{(2)}_i$ are 51-dimensional vectors that correspond to early and later arrival patterns for day $i$. Assuming multivariate normality, the optimal mean squared error forecast of $x^{(2)}_i$ given $x^{(1)}_i$ is
\begin{equation} \label{eq:meanfor}
E(x^{(2)}_i|x^{(1)}_i) = \mu_2 + \Sigma_{21}\Sigma^{-1}_{11}(x^{(1)}_i - \mu_1),
\end{equation}
 corresponding to partitioning of  the mean and covariance matrix of the full vector:
\[ \mu^t=(\mu^t_1,\mu^t_2),\; \Sigma= \begin{bmatrix} \Sigma_{11} & \Sigma_{12}\\ \Sigma_{21} & \Sigma_{22}. \end{bmatrix}\]
 % The corresponding partitions for the mean and covariance matrix are following
%\[ \mu^t=(\mu^t_1,\mu^t_2),\; \Sigma= \begin{bmatrix} \Sigma_{11} & \Sigma_{12}\\ \Sigma_{21} & \Sigma_{22} \end{bmatrix}\]
%Assuming multivariate normality, the best mean squared error forecast of $y^{(2)}_i$ using $y^{(1)}_i$ is
%\begin{equation}
%E(y^{(2)}_i|y^{(1)}_i) = \mu_2 + \Sigma_{21}\Sigma^{-1}_{11}(y^{(1)}_i- \mu_1)
%\end{equation}

We compare the forecast performance of six covariance estimators (SC (HP, Fused, Trend), CSCS, HSC, and S)
 by using training and test datasets described above. The sample mean and covariance matrix are computed  from the  training data for each $T$ . Using (\ref{eq:meanfor}), the  51 first half of a day arrival counts were used to forecast the second half of the day arrival counts. For each time interval $t= 52,\dots,102$, we define the forecast error (FE) by the average
\[FE_t= \frac{1}{239-T}\sum_{i=T+1}^{239}|\hat x_{it} - x_{it}|,\]
where $x_{it}$ and $\hat x_{it}$ are the observed and forecast values, respectively \citep{Huang2006}.
\begin{table}[t!] \centering
  \caption{Number of times (out of 51) each estimation method achieves the minimum forecast error for training data size 205, 150, 100, 75.}
  \label{t:ccfe}
\begin{tabular}{@{\extracolsep{5pt}} ccccc}
\\[-1.8ex]\hline
\hline \\[-1.8ex]
& \multicolumn{4}{c}{Training data size} \\
\cline{2-5} \\[-1.8ex]
 Method&205 & 150 &100 &  75 \\
\hline \\[-1.8ex]
SC-HP & 1 & 5 & 9 & \textbf{20} \\
SC-Fused & 3 & 7 & 2 & 3 \\
SC-Trend & 7 & 8 & \textbf{16} & 1 \\
CSCS & 3 & 8 & 10 & 15 \\
HSC & 12 & \textbf{13} & 14 & 12 \\
S & \textbf{25} & 10 & - & - \\
\hline \\[-1.8ex]
\end{tabular}
\end{table}
Table~\ref{t:ccfe} reports the number of times each of the six forecast methods has the minimum forecast error values out of the total 51 trials.  The maximum of the number of  times the method achieves the minimum forecast error in each column is  in bold. When $T = 205$ and the training data size is larger than the number of variables, the forecast based on the Sample covariance matrix performs the best in terms of the number of times it achieves the minimum forecast error. For $T =150$, the HSC is a  bit better than the sample covariance matrix. However, as the training data size decreases, the forecasting ability of the SC algorithm increases. In particular, the  SC-Trend and HP report the best result in terms of the number of times they achieve the minimum forecast error for $T = 100$ and $T =75$, respectively. Most of the result in Table~\ref{t:ccfe} is supported by the aggregate forecast errors reported in Table~\ref{t:afe}, where aggregate forecast error is the sum of forecasted errors over $t = 52,\dots,102$. The minimum aggregate forecast error the method achieves is in bold. The discrepancies between Table~\ref{t:ccfe} and Table~\ref{t:afe} can be explained by looking on Figure~\ref{fig:aeplot}, which illustrates  a plot of $FE_t$ for varying values of the training data size. For example, for $T = 150$ the HSC achieves the minimum forecast error the most in terms of the number of times, however SC-Fused is the lowest in terms of the aggregate forecast error. This discrepancy explained from the top right plot of  Figure~\ref{fig:aeplot}, where it can be seen that when the $FE_t$ of HSC is lowest, the $FE_t$  of SC-Fused does not concede to much, but when the $FE_t$  of SC-Fused is the lowest, HSC takes higher values, which forces the aggregate error of SC-Fused to be lower than the error of HSC,

\begin{table}[ht!] \centering
  \caption{Aggregate forecast error for each estimation method for training data size 205,150,100,75}
  \label{t:afe}
\begin{tabular}{@{\extracolsep{5pt}} ccccc}
\\[-1.8ex]\hline
\hline \\[-1.8ex]
& \multicolumn{4}{c}{Training data size} \\
\cline{2-5} \\[-1.8ex]
Methods& 205 & 150 & 100 & 75 \\
\hline \\[-1.8ex]
SC-HP & $403.555$ & $34.093$ & $24.186$ & $\mathbf{9.363}$ \\
SC-Fused & $377.938$ & $\mathbf{25.299}$ & $31.377$ & $44.403$ \\
SC-Trend & $371.936$ & $31.981$ & ${23.202}$ & $23.565$ \\
CSCS & $307.096$ & $38.130$ & $28.578$ & $13.799$ \\
HSC & $151.817$ & $28.299$ & $\mathbf{22.917}$ & $11.214$ \\
S & $\mathbf{111.276}$ & $42.432$ & $-$ & $-$ \\
\hline \\[-1.8ex]
\end{tabular}
\end{table}

\begin{figure}[t!]
\centering
\includegraphics[width=13cm,height=9.1cm]{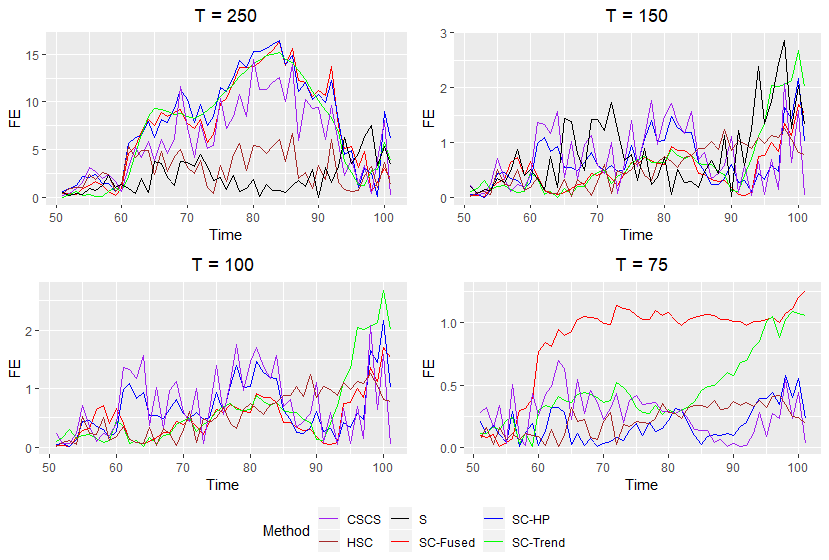}
\caption{ Forecast Error for each estimation method for training data size 205,150,100,75}
\label{fig:aeplot}
\end{figure}

An \textsf{R} \citep{rcore} package, named \texttt{SC}, is available on Github repository \citep{aram2019}. The core functions are coded in \textsf{C++}, allowing us to solve large-scale problems in substantially less time.
%%%%%%%%%%%%%%%%%%%%%%%%%%%%%%%%%%%%%%%%%%%%%%%%%%%%%%%

%%%%%%%%%%%%%%%%%%%%%%%%%%%%%%%%%
\section{Conclusion}

This paper proposes a novel penalized likelihood approach for smooth Cholesky-based covariance (inverse)
 estimation for longitudinal data or when a natural ordering among the variables is available.
  We  reparameterize the normal likelihood in terms of the standard Cholesky factor \citep{khare2016} and rely on
  fused-type Lasso penalties to formulate jointly convex objective functions. A block coordinate descent algorithm
is proposed to minimize the objective function. We establish  convergence of the algorithm which always leads to
positive-definite estimate of the covariance matrix.
 A goal  of the study
  was to explore the  connection between the local stationarity of a time series and  smoothness of the subdiagonals of the
  Cholesky factor of its (inverse) covariance matrix.
  A connection between subdiagonal smoothness of the standard Cholesky factor $L$ and the modified factor $T$ is established.  The
  performance of our methodology is illustrated via various simulations and two  datasets.

\section{Supporting Information}
Additional Supporting Information may be found online in the supporting information tab for this article.

\section{Data Availability Statement}
Cattle data that supports the findings of this study are available as supporting information online, while Call Center data are available on request from the corresponding author. The latter data are not publicly available due to privacy or ethical restrictions.
%%%%%%%%%%%%%%%%%%%%%%%%%%%%%%%%%%%%%5

\clearpage

\appendix
\section{Proof of Lemma~\ref{l:l1}} \label{ap:A3}

\paragraph{(a):} We use the selection matrices $K_i$ which are $(p - i) \times p^2$ submatrices of the $p^2 \times p^2$ identity
 matrix with row indices in $I_j$ such that $L^i = K_iV$  .  Then it is evident that  $V = \sum_{i=0}^{p-1}K_i^tL^i$  and
 a compatible partition of $B$ leads to
 %\textbf{ $K$ partiton it into $p^2$  matrices with index $0_{i}$, similarly for C add verbel discription of $B$}
 \begin{equation} \label{eq:e4}
\begin{aligned}
tr(LSL^{t}) & = V^tBV = \sum_{i=0}^{p-1}(L^i)^tK_iB\sum_{j=0}^{p-1}K^t_jL^j\\
	        & =  \sum_{i=0}^{p-1}\sum_{j=0}^{p-1}(L^i)^tK_iBK^t_jL^j = \sum_{i=0}^{p-1}\sum_{j=0}^{p-1}(L^i)^tB_{ij}L^j.
%	      & =\sum_{i=0}^{p-1} (L^i)^tK_iBK^t_i L^i + \sum_{i \neq j}, (L^i)^t K_iBK^t_jL^j,\\
%	      & = \sum_{i=0}^{p-1} \Big [ (L^i)^tB_{ii} L^i + (L^i)^t \sum_{i \neq j} B_{ij}L^j \Big ]  %(L^0)^tB_{00}L^0 + 2(L^0)^tB_{0\eta}L^{\eta} + (L^{\eta})^tB_{\eta \eta}L^{\eta}
\end{aligned}
\end{equation}
 Note that the submatrix   $B_{ii}=\mbox{diag}(S_{1,1}, S_{2,2},\dots, S_{p-i,p-i})$ is diagonal with positive entries, and the $(p-i) \times (p -j)$ matrix $B_{ij}$ has nonzero values in the $(( 1 + j - i + k), (1+k)),\, 0 \leq k \leq \mbox{min}(p - 1 + i - j, p - i -1)$ entries, which correspond to the diagonal of the submatrix $S[(1 + j -i) : (p-i), 1 : (p-j)]$ of $S$.

%$B_{ij}=K_iBK^t_j$ corresponds to selecting the rows and columns of $B$ with indices in $I_i$ and $I_j $, respectively.

%Then by partitioning the matrices $B_{0\eta}$ and $B_{\eta \eta}$ based on subdiagonal indices $I_j$,the simple algebra shows that (\ref{eq:e4}) can be written as
%\begin{equation} \label{eq:a5}
%tr(LSL^{t})= (L^{0})^tB_{00}L^{0}+\sum_{i=1}^{p-1} \Big [(L^{i})^tB_{ii}L^{i} + (L^{i})^t(\sum_{\substack{j=0\\j \neq i}}^{p-1}B_{ij}L^{j}+B_{i0}L^{0}) \Big ],
%\end{equation}%=f_0(L^{0})+\sum_{i=1}^{p-1}f(L^i,L^{-i})
%where $B_{ij}$ is .

\paragraph{(b):} From  rewriting (\ref{eq:e4})as
\begin{equation} \label{eq:e5}
\begin{aligned}
tr(LSL^{t}) & =\sum_{i=0}^{p-1}\Big [ (L^i)^tK_iBK^t_i L^i + \sum_{i \neq j} (L^i)^t K_iBK^t_jL^j\Big],\\
	      & = \sum_{i=0}^{p-1} \Big [ (L^i)^tB_{ii} L^i + (L^i)^t \sum_{i \neq j} B_{ij}L^j \Big ]  %(L^0)^tB_{00}L^0 + 2(L^0)^tB_{0\eta}L^{\eta} + (L^{\eta})^tB_{\eta \eta}L^{\eta}
\end{aligned}
\end{equation}
 the desired result follows from substituting into the objective function (\ref{eq:e2}) and noting that $|L| = \sum_{j =1}^p \log L^0_j$.
\paragraph{(c):} Since $B_{ii}$ is positive definite, then $Q_i(\cdot)$ as the sum of strictly convex and convex functions, is strictly convex \citep{Boyd2004}.

\section{Proof of Lemma~\ref{l:step7}} \label{ap:A2.1}
\paragraph{(a):} The derivative with respect to $x$ of the quadratic form in (\ref{eq:h0})  is
\begin{equation} \label{eq:upd}
-2 \sum_{i=1}^p\frac{1}{x_i}e_i + 2 C_0x + 2y_0 = 0,
\end{equation}
where $e_i$ is the $p-$vector with $i$th element equal to 1 and 0 otherwise. By construction, the
 $C_0$ matrix is diagonal and the first element of   $ y_0$ is 0, so that the first identity in (\ref{eq:upd}) is
\[ - \frac {1}{x_1} + (C_0)_{1,1}x_{1}=0 \Rightarrow x_1=1/\sqrt{(C_0)_{1,1}}.\]
Similarly, for the rows, $i = 2,\dots,p$ we have
\[-\frac{1}{x_{i}}+(C_0)_{i,i}x_{i}+ (y_0)_i=0, \]
where its non-negative solution is as given in (\ref{eq:iterd}).

\paragraph{(b):} \label{p:l2pb}
In (\ref{eq:h1}), $C_i$ is a diagonal matrix with positive entries, setting  $\tilde y_i=-C_i^{-1/2}y_i$
  and  completing the square, then finding $x^*_i$ is equivalent to solving a generalized lasso problem:
  % and $\tilde L^{i} =-B_{ii}^{1/2}L^{i},\;\tilde D^{(i)}=D^{(i)}B_{ii}^{-1/2}$ ,
 \begin{equation} \label{eq:genl}
 \min_{x} \Big \{ \|C^{1/2}_ix-\tilde y_i\|^2_2+\lambda \|Dx\|_1 \Big \},
 \end{equation}
which has a unique solution \citet{tibshirani2011}.

\paragraph{(c):}\subparagraph{(1):} Proof is similar to the transformation in part (b).
\subparagraph{(2):} Setting the derivative of  $h_i(x|y_i)$ to zero and solving for $x_i$ gives
 \[x^*_i = -\frac{1}{2} (C_i+\lambda (D^tD))^{-1}y_i.\]
 The matrix inverse can be computed in $O(p-i)$ flops \citet{golub1996}, since $C_i$ is diagonal and $D^tD$ is a tridiagonal matrix.
 Here $p-i$ is the length of the vector $x_i, i = 1,\dots,p-1$.

\paragraph{(d):} Proof of the lemma is similar to \citet[Proposition 1]{friedman2007}, thus omitted.

%\section{Estimation of $x^*_i$ for the special H-P  penalty} \label{ap:C}
%We estimate $x^*_i$  for the special H-P panalty:
% \[P(x)= \sum_{i=1}^{p-1}\|Dx\|_2^2,\]
% where $D$ is the second order difference matrix. Setting the derivative of  $h_i(x|y_i)$ to zero and solving for $x$ gives
% \[x = -\frac{1}{2} (C_i+\lambda (D^tD))^{-1}y_i.\]
% The matrix inverse can be computed in $O(p-i)$ flops \citet{golub1996}, since $C$ is diagonal and $D^tD$ is a tridiagonal matrix.
% Here $p-i$ is the length of the vector $x, i = 1,\dots,p-1$.

%\section{Sparse SC} \label{ap:D1}
% When in addition to smoothing, the interest is  also in sparsity, then for each $1 \leq i \leq p-1$, we add the sparse term  $\lambda_1\|x\|_1$ to (\ref{eq:h1})
%\begin{equation} \label{eq:sfused}
%\tilde h_i(x|y)= h_i(x|y) + \inf_{x \in \mathcal{R}^{p-i}} \lambda_1\|x\|_1
%\end{equation}
%with the following  proposed solution:
%
%\begin{lemma} Denote the solution of  (\ref{eq:sfused}) for $\lambda_1=0$ and $\lambda_2 \geq 0$ as $\hat x(0,\lambda_2)$. Then, its solution for $\lambda_1 >0$ is
%\[ \hat x_j(\lambda_1,\lambda_2)=sign( \hat x_j(0,\lambda_2))(| \hat x_j(0,\lambda_2)|-\frac{1}{2}(C^{-1})_{j,j}\lambda_1)_{+}\]
%\end{lemma}

\section{Proof of Theorem~\ref{l:conv}} \label{ap:E}
\paragraph{(a):}
Recall that $L^{-0}=[(L^1)^t,\dots,(L^{p-1})^t]^t$ where $L^i$ is the vector of $i$th subdiagonal. To
 make a change of variables in terms of  difference of successive subdiagonal terms, define
 $\theta =[(\theta^1)^t,\dots,(\theta^{p-1})^t]^t$, where $\theta^j_1=L^{j}_1, \theta^j_i = L^{j}_i - L^{j}_{i-1},\;\mbox{for each}\;\;1\leq j \leq p-1,\; i =2, \dots,p-j $. Then, we have $L^{-0}= A \theta$ where $A \in R^{\binom{p}{2} \times \binom{p}{2}} $ is a block diagonal matrix
  where the $i$th $(1 \leq i \leq {p-1})$ block is a $(p-i) \times (p-i)$ lower triangular matrix with
   ones as the nonzero entries. Substituting for $L^{-0}$ in $Q(L)$, we get
  \begin{equation}
 Q(L)=(L^0)^tB_{00}L^0 + 2(L^0)^t B_{0\,-0}A \theta + \theta^t A^t B_{-0 \,-0}A \theta - 2 \sum_{i=1}^p \log L^0_i + \lambda \sum_{j =1}^{p-1}\sum_{i=2}^{p-j}|\theta^j_i|,
 \end{equation}
 where $B_{-0-0}$ is the submatrix that selects the rows and columns of $B$ with indices in $\{I_{-0}, I_{-0} \}$.
Next, we rewrite
\begin{equation} \label{eq:trsc}
Q(L) = x^tMx - \sum_{i =1}^{p}\log x_i + \sum_{j \in C}|x_i|,
\end{equation}
 where $x =[L^0,\theta]^t$ , $M= \tilde A^t \tilde B \tilde A$ ,
 %\[M= \begin{bmatrix}
%B_{00}&B_{0\eta}A\\
%(B_{0\eta}A)^t& A^tB_{\eta \eta}A\\
%\end{bmatrix},\;\;\tilde
\[\tilde A =\begin{bmatrix} I & \mathbf{0}\\
						\mathbf{0}^t&A\end{bmatrix},\tilde B= \begin{bmatrix}
B_{00}&B_{0\, -0}\\
(B_{0\, -0})^t& B_{-0 \, -0}\\
\end{bmatrix},\]
and the set $C=\{i|x_i = \theta^j_k,\;1 \leq j \leq p-1,\; 2 \leq k \leq p-j\}$ corresponds to the indices
of the difference terms in $\theta$.
%Then
%\[\;\; \mbox{and}\;\; Q(L) =x^t M x - 2 \sum_{i=1}^{p} \log x_i + \lambda \sum_{j \in C} |x_j|,\]

The matrix $\tilde B$ is positive semi-definite, since it is a submatrix of the
positive semi-definite matrix $B$ obtained by selecting specific rows and columns . Therefore, from
 \citet[Observation 7.1.8]{Horn2012} the matrix $M$ is positive semi-definite
            and can be written as $M=E^tE$ \citep[Chapter 7]{Horn2012} which establishes  the
            equivalency of $Q(L)$ and ($\ref{eq:tc}$). Since the diagonal elements of the
            sample covariance matrix is assumed to be positive, then $B$ and hence $E$ do not have 0 columns.

We note that (\ref{eq:trsc}) is not a fully form of  (\ref{eq:f2}), since the $\ell_1$ penalty reformulation involves $p-i - 1$ of the $p-i$ components in each subdiagonal $L^i$. However, with the transformation similar to \citet[Lemma 2.4]{rojas2014} easily formulate (\ref{eq:trsc}) as (\ref{eq:f2}).
%From \citet[Observation 7.1.8]{Horn2012}, if $\tilde B$ is positive semi-definite then the positive semi-definiteness of $M$ guaranteed. Since $B$ is positive semidefinite and $\tilde B$ can be considered as submatrix of $B$ by selecting specific rows and columns then positive semidefinitness follows from the definition. Thus $M$ is positive semidefinite.

 % and choosing the subsets $S$ and $S^c$ to correspond selecting $\theta$ and $L^0$ ,respectively, we have the transformed version of objective function as $h(x)$
%\begin{equation} \label{eq:alt}
%h(\tilde \theta)= \tilde \theta^t E^t E \tilde \theta - 2 \sum_{i \in S^c} \log \tilde \theta_i + \lambda \sum_{i \in S} |\tilde \theta_i|
%\end{equation}

%Since $Q_{SSC}(L)$ is proper convex function, then Assumption $A5^*$  of \citet{khare2014} satisfied and the problem (\ref{eq:alt}) is convergent. Suppose $\tilde \theta^*$ is a solution of (\ref{eq:alt}), then initial problem can be recovered by $(L^{\eta})^* = A \tilde \theta^*$

\paragraph{(b):} %We start the algorithm with the covariance and initial Cholesky factor $L$ matrices such that both have positive diagonal elements.
  % From the SSC algorithm, the update of diagonal occurs by minimizing (\ref{eq:h0}). Thus from part (a) of Lemma~\ref{l:step7}, the diagonal elements of $(L)^k$ after each iteration is positive and belong to $\mathcal{L}^p$, where $\mathcal{L}^p$ is the space of all lower diagonal matrices with positive diagonal elements. %The proof of the convergence of SSC algorithm to a stationary point of $Q$ follows from the general results on block coordinate descent algorithms \citep{tseng2001,bertsekas2016}, since  the update of diagonal and every  subdiagonal optimizes strictly convex function. Note that $Q(L) \rightarrow \infty$, when $L_{i,i} \rightarrow 0$ and the stationary point belongs to $\mathcal{L}^p$. However, this does not mean that iterates  itself converges.
We show  the convergence of iterates produced by Algorithm~\ref{a:SSC} to a global minimum by invoking  \citet[Theorem 2.2]{khare2014}.
%First the convergence of fused lasso is provided and the convergence of $\ell_1$ trend and $HP$ is discussed in Appendix~\ref{ap:E11}.

From Part (a) of the Theorem, there exist matrix $E$ with no 0 columns such that (\ref{eq:tc} holds and Lemma~\ref{lemma-np}  shows an existence of an uniform lower bound for $Q(L)$. Thus, to show convergence,  it suffices  to show that the assumption (A5)* \citet[page 6]{khare2014} is satisfied or the level set of $Q(L)$, $\{L|Q(L) \leq Q(L^{(0)})\}$ is bounded. The latter property follows from the coercive property of the $Q(L)$ established in Lemma~\ref{lemma-np}, since the level sets of coercive function are bounded \citep{bertsekas2016}.

\section{Proof of  Lemma \ref{lemma-np}}
 In objective function $Q(L)$, $L \in \mathcal{L}_p$ and the eigenvalues of a lower triangular matrix are its diagonal elements, then
 from the well-known inequality $\log x \leq x-1,\; x>0$ it follows that $$\sum_{j=1}^p \log L^0_j \leq \sum_{j=1}^p (L^0_j -1) \leq (L^0)^t \mathbf{1}_p.$$
 Thus
\[
\begin{aligned}
Q(L) & \geq (L^0)^t B_{00}L^0 + 2 (L^0)^t B_{0\, -0}L^{-0} + (L^{-0})^t B_{-0\, -0} L^{-0} - 2 (L^0)^t \mathbf{1}_p \\
& \eqtext{(*)} \|B^{1/2}_{-0\, -0} L^{-0} + B_{-0\, -0}^{-1/2}B_{-0\, 0} L^0\|^2_2 +(L^0)^t(B_{00}-B_{0 \,-0}B^{-1}_{-0\, -0}B_{-0\, 0})L^0- 2(L^0)^t \mathbf{1}_p\\
&\geqtext{(**)} \|B^{1/2}_{-0\, -0} L^{-0} + B_{-0\, -0}^{-1/2}B_{-0\, 0} L^0\|^2_2 + \|K^{1/2}L^0 - K^{-1/2} \mathbf{1}_p\|^2_2 - \mathbf{1}^t_pK \mathbf{1}_p\\
& \geqtext{(***)} (\|B^{1/2}_{\eta \eta} L^\eta\| -\| -B_{\eta \eta}^{-1/2}B_{\eta 0} L^0\|)^2 +(\|K^{1/2}L^0 \|- \|K^{-1/2} \mathbf{1}_p\|)^2 - \mathbf{1}^t_pK \mathbf{1}_p
 \geq - \mathbf{1}^t_pK \mathbf{1}_p > - \infty  ,
\end{aligned}
\]
where the equality in (*) follows from completing the square by adding and subtracting $\| B^{-1/2}_{-0\, -0}B_{-0\,0}L^0\|^2_2$  and writing
$ (L^0)^t B_{0\, -0}L^{-0}=(L^0)^t B_{0\, -0}B^{-1/2}_{-0\, -0}B^{1/2}_{-0\, -0}L^{-0}$.
The inequality in (**) follows by completing the middle term as square and noting that $K=B_{00}-B_{0\, -0}B^{-1}_{-0\, -0}B_{-0\, 0}$ is
 positive semi-definite (the Schur complement of the positive semi-definite matrix $\tilde B$) and (***) is based on the triangle inequality $\|x\| - \|y\| \leq \|x-y\|$.
 
It follows from (**) and (***) that $Q(L) \rightarrow \infty$ as any subdiagonal $\|L^j\| \rightarrow \infty$ , and that if any diagonal element $L^0_j =0$ then $Q(L) \rightarrow \infty$. Therefore, any global minimum of $Q(L)$ has a strictly positive values for $L^0$ and hence any global minimum of $Q(L)$ over the open set $\mathcal{L}_p$ lies in $\mathcal{L}_p$. Here, $\mathcal{L}_p$ is open in the set of all lower triangular matrices.
Moreover, from the discussion above the function $Q(L)$ is coercive, i.e. if $\|[L^0,L^{-0}]\| \rightarrow \infty$, then $Q(L) \rightarrow \infty$.

\section{Convergence of $\ell_1$-trend filtering and HP} \label{ap:E11}
The convergence proof of trend filtering follows the same steps as described in previous section. For this case, the change of variables occurs by taking $A \in R^{\binom{p}{2} \times \binom{p}{2}} $ as a block diagonal matrix
  where the $i$th $(1 \leq i \leq {p-1})$ block is a $(p-i) \times (p-i)$ lower triangular matrix with
   the sequence $1,\dots, p -j$ as a nonzero elements in $j$th column \citep[Section 3.2]{kim2009}. The rest of the proof is similar to Appendix~\ref{ap:E}, thus omitted.

%It can be seen that the (b) part of Theorem~\ref{l:conv} holds for all three discussed penalty forms.
 For the convergence of HP, we note that for this case $Q(L)$ is convex differentiable function and the existing literature can be used to show convergence. For example see \citet{Luo1992}.

\section{Proof of Lemma~\ref{l:cost}} \label{ap:D}
%\begin{proof}
We use ideas similar to the \citet{friedman2010a, cai2011, khare2016}. We start by considering two cases
\paragraph{Case 1 ($n \geq p$)}
Each iteration of SC Algorithm sweeps over diagonal and subdiagonal elements. Thus, update of the diagonal consists of estimating $y$ and then computing diagonal using Lemma~\ref{l:step7}. From the discussion provided before the Theorem~\ref{l:l1}, recall that matrix $B_{ii}$ is diagonal and $B_{ij}$ has $p-j+1$ nonzero elements located in separate columns, for $0 \leq i,j \leq p,\; i \neq j$.
 Thus the complexity of computing $y_0$ in SC algorithm is
\[\sum_{j=1}^{p-1}(p-j)=p(p-1)-\frac{p(p-1)}{2} \approx O(p^2)\]
From the Lemma~\ref{l:step7}, the computational cost of estimating diagonal is $O(p)$. Therefore the cost of diagonal update can be done in $p(p+1)/2$ steps.

 The update of each subdiagonal consist of computing $y_i,\; 1 \leq i \leq p-1$  and estimating the subdiagonal in SC algorithm. Thus, the cost of estimating $y_i$ is
 \[p + \sum_{j=0}^{i-1}(p-j) + \sum_{j=i+1}^{p-1}(p-j)=\frac{p(p-1)}{2} -p - i^2\]
 and since each iteration sweeps over $p-1$ subdiagonals we have
 \[\sum_{i=1}^{p-1}\frac{p(p-1)}{2} -p - i^2 \approx O(p^3).\]

\paragraph{Case 2 ($n < p$)} We use similar technique as in \citet[Lemma C.1]{khare2016}. Note that, since $S = Y Y^t /n $, where $Y \in R^{p \times n}$ matrix, then $B = S \otimes I_p = (Y \otimes I_p)(Y \otimes I_p)^t/n = AA^t $, where $A = (Y \otimes I_p)/ \sqrt{n} $. Moreover $B_{jk} = A_{j\cdot}A^t_{\cdot k}$, where $A_{j\cdot}$ is submatrix whose rows were selected from index $I_j$. Recall $V = vec(L)$  and let $r(V) = A^tV \in R^{pn}$, which takes $O(np^2)$ iterations, due to sparsity structure of $A$. Given initial value $V^{0}$, we evaluate $r(V^{0}) = A^t V^{0}$ and keep truck of $A^tV^{\mbox{current}}$.
If $V$ and $\tilde V$ differ only in one block coordinate $k$, then
\begin{equation} \label{eq:comp2}
 (A^tV)_j = \sum_{j=0}^{p-1}A_{\cdot j} \tilde L^j = \sum_{j=0}^{p-1}A_{\cdot j} L^j + A_{\cdot k}(\tilde L^k - L^k),
 \end{equation}
for $1 \leq j \leq np$. Therefore it takes $O(np)$ computations to update $A^tV$ to $A^t\tilde V$. Hence, after each block update in SC algorithm, it will take $O(np)$ computations to update $r$ to its current value.
Thus, the computation of $y_i$ can be transformed into
\begin{equation} \label{eq:comp1}
\sum_{j \neq i} B_{ij}(L^j) = \sum_{j =1}^p B_{ij}L^j - B_{ii}L^i = A_{i\cdot}\sum_{j=1}^pA^t_{\cdot j}L^j - B_{ii}L^i,
\end{equation}
for $0 \leq i \leq p-1$. It follows the update of k'th block in (\ref{eq:comp2}), consequently in (\ref{eq:comp1}) takes $O(np)$ computations. Hence one iteration will take $O(np^2)$ computations.

\section{Proof of Lemma \ref{l:TL}}  \label{ap:E1}
\paragraph{(a):}
From (\ref{eq:e1}), for any two elements in $i$th subdiagonal $u,v$
\[|L^i(u) - L^i(v)| =  |\frac{T^{i}(u)}{\sigma(u)} - \frac{T^{(i)}(v)}{\sigma(v)}| = |\frac{T^{i}(u) - T^{i}(v)}{\sigma(v)} + \frac{T^{i}(u)\Delta_{vu}(\sigma)}{\sigma(u)\sigma(v)}|, \]
where $\Delta_{vu}(\sigma) = \sigma(v) - \sigma(u)$. The simple algebra shows that
\begin{equation} \label{eq:TL}
 |L^i(u) - L^i(v)| \leq \frac{1}{c}|\Delta_{uv}(T^i)| + \frac{1}{c^2}|T^i(u)||\Delta_{vu}(\sigma)|
\end{equation}
\paragraph{(b):}The bounded total variation of $L^i$ follows from the fact that it is product of two functions of bounded total variation \citep[Theorem 2.4]{Grady2009} and (\ref{eq:tv}) follows from summing (\ref{eq:TL}) over $u,v \in [0,1]$
\[ \begin{aligned}
|L^i(u) - L^i(v)| &=  |\frac{T^{i}(u)}{\sigma(u)} - \frac{T^{(i)}(v)}{\sigma(v)}|
 =|\frac{T^{i}(u)}{\sigma(u)} - \frac{T^{i}(u)}{\sigma(v)} +\frac{T^{i}(u)}{\sigma(v)} -   \frac{T^{(i)}(v)}{\sigma(v)}|\\
 & \leq \frac{|T^i(u)||\sigma(v) - \sigma(u)|}{\sigma(u)\sigma(v)} + \frac{|T^i(v) - T^i(u)|}{\sigma(v)}
\end{aligned}
\]

%Here we adapt  $L^{i}(u) = L^{i}(j/N) = L^{i}_{uN}$ notation for the  $j= uN$th element to highlight the fact that the $i$th subdiagonal is viewed as a function of time.

 \section{Proof of Proposition~\ref{p:p1}} \label{ap:prop}
We say that the matrix $A \in R^{p \times p}$ belongs to the class $\mbox{TV}(R^{p \times p})$ if its diagonal and subdiagonals are functions of bounded variation. The following notation introduced in  \citet[Chapter 1.2.8]{golub1996} simplifies the discussion of the proof. For $L \in R^{p\times p}$ we introduce the matrix $D(L,i) \in R^{p \times p}$, which has the same $i$th sub(sup)diagonal as $L$ and 0 elsewhere.
%:
%\[D(L,i)_{lk} = \begin{cases} L_{lk} & k = l - i,\, 1 \leq l,k \leq p\\ 0 & \mbox{otherwise}.
%\end{cases}\]
Clearly, if $A \in TV(R^{p \times p})$ then $D(A,i) \in TV(R^{p \times p}),\; 0 \leq i \leq p -1$. For the lower triangular matrix $L$ we have
\[L = \begin{sbmatrix} {D(L,0)} L_{11} & 0 & \hdots & 0 \\
                                0        & L_{22} & \vdots & \vdots \\
				\vdots & \vdots & \ddots& 0\\
				0       & \hdots   & 0 &   L_{pp}
	\end{sbmatrix}  + \begin{sbmatrix} {D(L,1)} 0 & \hdots & \hdots & 0 \\
				L_{21} & 0 &\vdots& \vdots\\
				\vdots      & \ddots   &    \vdots & 0 \\
				0              &   0        & L_{p, p-1} & 0
	\end{sbmatrix}  + \dots + \begin{sbmatrix} {D(L,p-1)} 0 & \hdots & \hdots & 0 \\
				\vdots & \vdots &\vdots& \vdots\\
				\vdots      & \vdots   &    \vdots & \vdots \\
				L_{p1}              &   0        & \hdots & 0
	\end{sbmatrix}\]
and
\[\Sigma = L^tL = ( D(L,p-1)+\dots + D(L,0))^t( D(L,p-1)+\dots + D(L,0)).\]
 From the structure of $D(L,i)$'s it can be seen that the $i$th subdiagonal of $\Sigma$ can be written as the sum of the $i$th subdiagonals of the following matrix products% the $i$th subdiagonal of the following sum of matrix products
\begin{equation} \label{eq:sigL}
\Sigma^i = \sum_{j=0}^{p-i-1}((D(L,j)^tD(L,j+i))^i,
\end{equation}
where from the position of degenerate values, the matrix product $D(L,j)^tD(L,j+i)$ has nonzero values on the $i$th subdiagonal and zero elsewhere. Moreover, nonzero values in the $i$th subdiagonal of $(D(L,j)^tD(L,j+i))^i = (L^i_{(p-j-i-1):(p -1)})^t L^{j + i}$.  Now, since the product of two functions of total bounded variation are of bounded variation and after adding and subtracting corresponding terms as in the proof of Lemma~\ref{l:TL}, we get
\[TV(((D(L,j)^tD(L,j+i))^i) \leq m_jK_{j+i} + m_{j+i}K_j\]
and the result follows from (\ref{eq:sigL}).

\paragraph{(b)}
We show the converse of the part (a), i.e. if $\Sigma \in \mbox{TV}(R^{p \times p})$ then  there exist a unique $L \in \mbox{TV}(T^{p \times p})$ and $\Sigma = L^t L$. The proof uses similar argument proposed in \citep[Lemma 2.8]{chern2000}.
Before introducing the main argument, we state the following lemma, which  will be used in the proof.

\begin{lemma} \label{l:tvl}
 If lower triangular matrices $G, M \in \mbox{TV}(R^{p \times p})$ then $A = GM \in \mbox{TV}(R^{p \times p})$
\end{lemma}
\begin{proof}
Using the matrix notation $(D(L,\cdots)$ introduced in part(a), it can be shown that the $i$th subdiagonal of the matrix product $A = GM$
can be written as
\[A^j = (GM)^j =\sum_{i = 0}^{j} (D(G,i)D(M, j - i))^j \]
and the result follows by recalling that the product of the functions of bounded variation is of bounded variation.
\end{proof}

The main argument consist in writing $\Sigma = \begin{bmatrix}  \hat \Sigma & b \\
							b^t  & \sigma^2_{pp} \end{bmatrix} $
and let $G_1 = \begin{bmatrix}  I_{p-1}  & b/\sigma_{pp} \\
						0 & \sigma_{pp}\end{bmatrix}$.
From the construction of $G_1$ and $\Sigma \in \mbox{TV}(R^{p \times p})$, it is easy to see that $G_1 \in \mbox{TV}(R^{p \times p})$ and $G^{-1}_1 \Sigma G^{-t}_1 = \begin{bmatrix} \Sigma_1 & 0 \\ 0 & 1 \end{bmatrix}$, where $\Sigma_1 = \hat \Sigma - bb^t/ \sigma^2_{pp}$. Clearly, $\Sigma_1 \in TV(R^{p-1\times p-1})$ since $\hat \Sigma \in TV(R^{p-1 \times p-1})$ by construction and $TV(\Sigma^i_1) = TV(\hat \Sigma^i -(bb^t)^i/\sigma^2_{pp}) \leq TV(\hat \Sigma^i) < \infty $. By repeating this procedure and using Lemma~\ref{l:tvl},  result follows.

 \section{Selection of Tuning Parameter} \label{s:tun}
We use BIC-like measure and cross-validation to choose the tuning parameter $\lambda$.
In particular,  the tuning parameter $\lambda$ is determined by choosing the minimum of BIC-like measure and CV over the grid. BIC is defined as:
 \[BIC(\lambda)=ntr(\hat L^{t}\hat LS)-n \log|\hat L^{t}\hat L|+\log n\times E ,\]
 where $E$ denoted the degrees of freedom, $n$ and $S$ are the sample size and covariance matrix, respectively. For example for the sparse fused lasso, $E$ corresponds to number of nonzero fused groups in $\hat L$ \citep{tibshirani2011}.

 For $K-$fold cross-validation, we randomly split the full dataset $\mathcal{D}$ into $K$ subsets of about the same size, denoted by $\mathcal{D}^{\nu},\; \nu=1,\dots,K$. For each $\nu$, $\mathcal{D}-\mathcal{D}^{\nu}$ is used to estimate the parameters and $\mathcal{D}^{\nu}$ to validate. The performance of the model is measured using the log-likelihood. We choose the tuning parameter $\lambda$ as a minimum of the $K-$fold cross-validated log-likelihood criterion over the grid.
 \[CV(\lambda)=\frac{1}{K}\sum_{\nu=1}^{K}\Big ( d_{\nu} \log |(\hat {L}^t_{-{\nu}} \hat {L}_{-{\nu}})^{-1}|+\sum_{I_{\nu}} y^t_i\hat {L}^t_{-{\nu}} \hat {L}_{-{\nu}}y_i \Big ),\]
 where $\hat {L}_{-{\nu}}$  is the estimated Cholesky factor using the data set $\mathcal{D}- \mathcal{D}^{\nu}$, $I_{\nu}$ is the index set of the data in $\mathcal{D}$, $d_{\nu}$ is the size of $I_{\nu}$, and $y_i$ is the $i$th observation of the dataset $\mathcal{D}$ .
 %The minimization of (~\ref{eq:h0}) does not depend on the penalty form, from Lemma~\ref{l:step7}, the only computational cost is in computing $y=\sum_{j\neq0}B_{0j}(L^j)$. From the discussion provided after proof of Theorem~\ref{l:l1}, matrix $B_{0j}$ has $p-i$ nonzero elements sitting in seperate columns. Thus using sparsity and the structure of $B_{0j}$, $\sum_{j\neq0}B_{0j}(L^j)$ can be computed in $O(p^2)$  . Since each iteration of SSC, computes $p$ sequence then the result follows.

\section{Additional Simulation} \label{s:adsim}
In this section we provide additional simulation results. Two different cases are considered. In the first case, matrix $T$ is full lower triangular matrix and subdiagonals are randomly chosen from the Cases (A-D) described in the Section~\ref{s:s3}. In the second case, matrix $T$ follows nonhierarchical structure, in a sense described in \citet{yu2017}. That is, in a full lower triangular matrix $T$, we enforce first and last $p/3$ subdiagonals admit nonzero values, drawn from uniform [0.1, 0.2] and positive/negative signs are then assigned with probability 0.5. The rest of $p/3$ subdiagonals admit zero value. See Figure ~\ref{fig:nhsnap} for an illustration. For the latter case, Sparse SC have been used for the estimation. That is we use two tuning parameters $\lambda_1$ to control sparsity and $\lambda_2$ smoothness, respectively.

%\begin{figure}[htb!]
%\centering
%\includegraphics[width=6cm,height=6cm]{nonhierarch}
%\caption{ Example of nonhierarchical matrix $T$ ( $p=50$).}
%\label{fig:nh}
%\end{figure}

For both cases, we consider settings when $p= 50, 150$ and $n = 100$, however because of the space limitation only $p = 150$ is reported. Each possible setting is repeated over 20 simulated datasets. The tuning parameters were chosen using cross-validation. Moreover, for the second case we compare the results from sparse SC (HP, Fused, Trend) estimator with CSCS and HSC using a receiver operating characteristic curve, or ROC curve.

We start by providing results for the first case. The Figure~\ref{fig:mp} plots the first four estimated subdiagonals of the full lower traingular matrix $T$,  using SC estimator. From the figure, the first two subdiagonals correspond to the Case B, the third to the Case D and the fourth to the Case C, respectively. Visually, the  SC-Fused captures the step function the best for the first subdiagonal. However, all three estimators failed to capture the stepwise linear structure of the second subdiagonal, but there is a significant improvement of SC estimator to capture the wiggliness of the Markov process in the third subdiagonal (SC-HP being the best) and smooth, slow time-varying structure of the fourth subdiagonal (SC-Trend being the best). Next we report the performance of three estimators using Frobenius and Infinity norm. Figure~\ref{fig:normplot} plots the results. Overall, for matrix $T$, SC-Trend filtering provides the lowest Frobenius and Infinity norm followed by SC-Fused.

\begin{figure}[H]
\centering
\includegraphics[width=13cm,height=9.5cm]{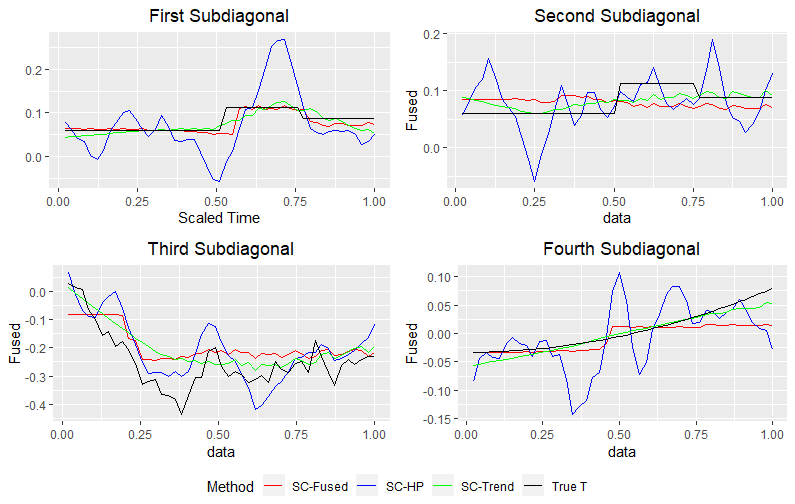}
\caption{ Estimated first four subdiagonals ( $p=150$).}
\label{fig:mp}
\end{figure}

\begin{figure}[H]
\centering
\includegraphics[width=11cm,height=6cm]{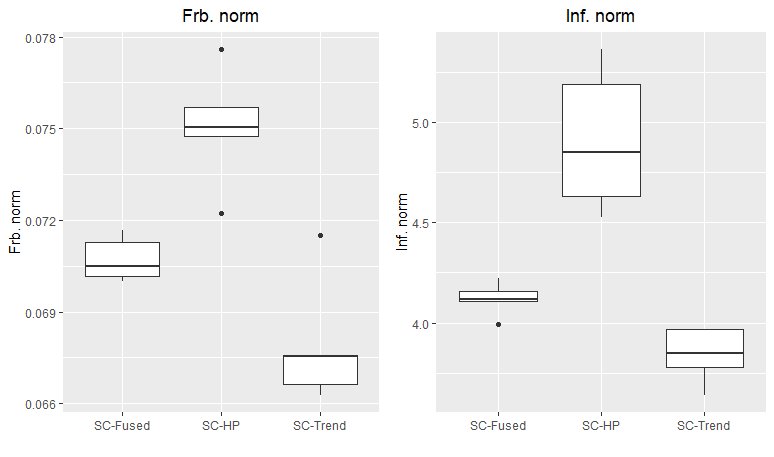}
\caption{ ROC curve for $p=150$.}
\label{fig:normplot}
\end{figure}

 \begin{remark} Relying on the result above, one can learn the lower triangular matrix $L$($T$) by considering the penalty form as an additional parameter to tune for each subdiagonal.
 \end{remark}

Now, we compare the performance of the SC with the CSCS and HSC on the support recovery, when the structure is non-hierarchical. Comparison is implemented using ROC curves. The ROC curve is created by plotting the true positive rate (TPR) against the false positive rate (FPR) at various penalty parameter settings \citep{friedman2010a}. Here, the ROC curve is obtained by varying around 60 possible values for the penalty parameter $\lambda_1$. For the SC-Fused, Trend and HP, the smoothing tuning parameter $\lambda_2$ is obtained from the cross-validation by fixing $\lambda_1$ in a given value. In applications, FPR is usually controlled to be sufficiently small, thus following \citet{khare2016}, the focus is on comparing portion of ROC curves for which FPR is less than 0.15. The comparison of ROC curves is implemented using Area-under-the-curve (AUC) \citep{friedman2010a}.

Table~\ref{tb:auc} reports the mean and the standard deviation (over 20 simulations) for the AUCs for SC (HP, Fused and Trend), CSCS and HSC when $p = 150$ and $n = 100$. The best  result is given in bold.

%Figure~\ref{fig:roc} plots the result. We find that the SC-Fused penalty provide similar results with HSC, followed by SC-Trend.
\begin{table}[!htbp] \centering
  \caption{Mean and Standard Deviation of area-under-the-curve (AUC) for 20 simulations for p = 150. }
  \label{tb:auc}
\begin{tabular}{@{\extracolsep{5pt}} lcc}
\\[-1.8ex]\hline
Method& Mean & Std. Dev \\
\hline \\[-1.8ex]
SC-HP&$0.068$ & $0.019$ \\
SC- Fused &$0.121$ & $0.023$ \\
SC- Trend &$0.104$ & $0.015$ \\
CSCS&$0.058$ & $0.007$ \\
HSC&$\mathbf{0.137}$ & $0.025$ \\
\hline \\[-1.8ex]
\end{tabular}
\end{table}

From the table above, it can be seen that HSC provides the best result. However, Figure~\ref{fig:nhsnap}, which captures snapshot of the graphical comparison of the estimated matrix $L$ for the five estimators, sheds more lights into characteristics of estimators. As can be seen, even though HSC provides the highest AUC for FPR less than 0.15, it fails to capture the zero gap between subdiagonals of matrix $L$ compare, for example, with SC-Fused, which provides the second best result in the Table~\ref{tb:auc}.

%\begin{figure}[H]
%\centering
%\includegraphics[width=8cm,height=6cm]{Truepositive_plot}
%\caption{ ROC curve.}
%\label{fig:roc}
%\end{figure}

\begin{figure}[H]
\centering
\includegraphics[width=9.5cm,height=8cm]{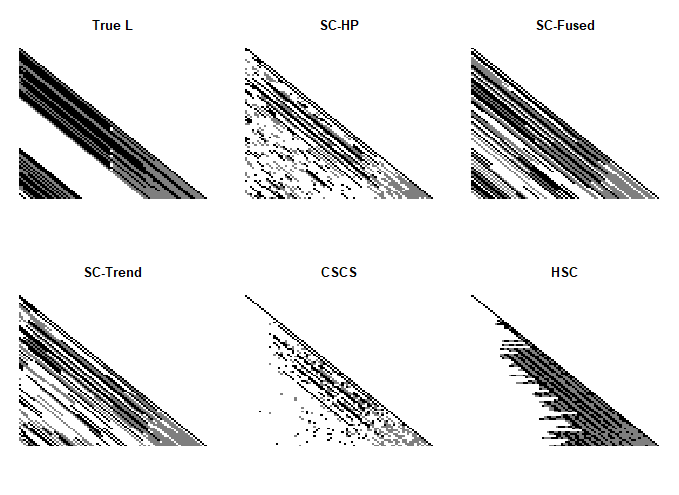}
\caption{ Comparison of snapshots for the simulated example for $p = 150$.}
\label{fig:nhsnap}
\end{figure}

\section{Cattle data: Additional Analysis} \label{s:cattlead}

Figure~\ref{fig:cfullplot} provides the plot of the first two subdiagonals using eight estimators descirbed in Section~\ref{s:cattle}.

\begin{figure}[H]
\centering
\includegraphics[width=14cm,height=6cm]{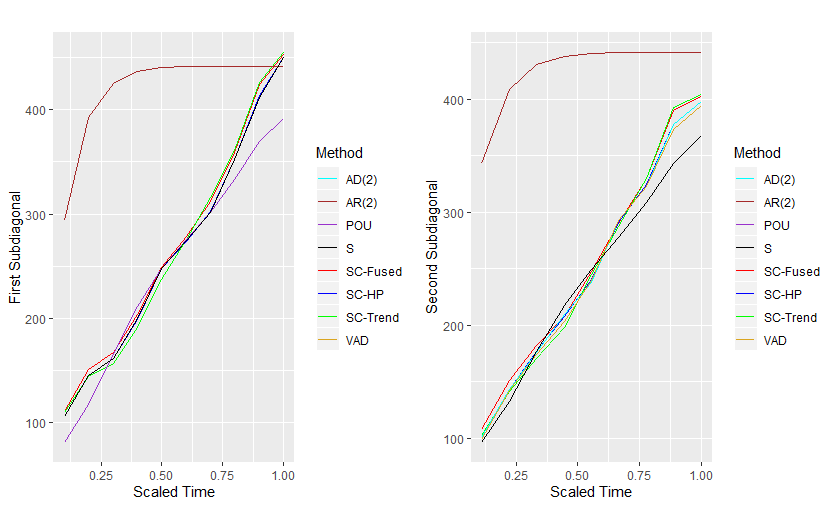}
\caption{Plots of estimated first and second subdiagonals of the covariance matrix for various estimation methods.}
\label{fig:cfullplot}
\end{figure}

\clearpage
\bibliography{chol_bib}
\bibliographystyle{te}

\end{document}